\documentclass[journal ]{new-aiaa}

\usepackage[utf8]{inputenc}
\usepackage{textcomp}
\usepackage{subcaption}

\usepackage{graphicx}
\usepackage{amsmath}
\usepackage{amsthm}
\usepackage[version=4]{mhchem}
\usepackage{siunitx}
\usepackage{longtable,tabularx}
\usepackage{ifsym}
\usepackage{mathtools}
\usepackage{algorithm}
\usepackage{algpseudocode}

\newtheorem{theorem}{Theorem}[section]
\newtheorem{definition}{Definition}[section]
\newtheorem{lemma}{Lemma}[section]
\newtheorem{corollary}{Corollary}[section]
\newtheorem{assumption}{Assumption}[section]

\DeclareMathOperator*{\argmax}{\arg\!\max}

\title{PoleStack: Robust Pole Estimation of Irregular Objects from Silhouette Stacking}

\author{Jacopo Villa \footnote{Professional Research Assistant, Laboratory for Atmospheric and Space Physics, 1234 Innovation Dr, Boulder, CO 80303} and Jay W. McMahon\footnote{Associate Professor, Colorado Center for Astrodynamics Research, 3775 Discovery Dr, Boulder, CO 80303}}
\affil{University of Colorado Boulder, Boulder, CO 80303, USA}
\author{Issa A. D. Nesnas\footnote{Principal Technologist, Jet Propulsion Laboratory, California Institute of Technology, 4800 Oak Grove Drive, Pasadena, CA 91109}}
\affil{Jet Propulsion Laboratory, California Institute of Technology, Pasadena, CA 91109, USA}

\begin{document}

\maketitle

\begin{abstract}
We present an algorithm to estimate the rotation pole of a principal-axis rotator using silhouette images collected from multiple camera poses. First, a set of images is stacked to form a single silhouette-stack image, where the object’s rotation introduces reflective symmetry about the imaged pole direction. We estimate this projected-pole direction by identifying maximum symmetry in the silhouette stack. To handle unknown center-of-mass image location, we apply the Discrete Fourier Transform to produce the silhouette-stack amplitude spectrum, achieving translation invariance and increased robustness to noise. Second, the 3D pole orientation is estimated by combining two or more projected-pole measurements collected from different camera orientations. We demonstrate degree-level pole estimation accuracy using low-resolution imagery, showing robustness to severe surface shadowing and centroid-based image-registration errors. The proposed approach could be suitable for pole estimation during both the approach phase toward a target object and while hovering.
\end{abstract}

\section{Introduction}

Estimating the rotational motion of a space object is a necessary task for various spaceflight applications. For missions to small celestial bodies (e.g., asteroids), the target body's rotation axis---referred to as \textit{pole} hereafter---must be accurately determined to define a body-fixed reference frame and compute surface-relative quantities supporting proximity operations, terrain-relative navigation, topography modeling, and characterization of the dynamical environment\cite{gaskell2008characterizing,palmer2022practical,adam2023stereophotoclinometry}. Similarly, for rendezvous, proximity operations, and docking (RPOD) scenarios at an artificial object, an accurate estimate of the rotation axis is needed to process measurements of the target surface, such as scans from a LIDAR sensor or reference points observed from a camera, and enable target-relative navigation by registering such measurements with a target-fixed reference frame\cite{opromolla2017pose}. Rotating bodies can be subdivided into principal-axis and non-principal-axis rotators (or tumblers), based on whether the rotation-axis direction is fixed or varies with respect to an inertial frame, respectively. In this work, we address pole estimation for principal-axis rotators, whereby the angular-momentum and angular-velocity vectors are both aligned with each other and with one of the principal axes of inertia. It has been shown that a large portion of small celestial bodies undergoes principal-axis rotation\cite{harris1994tumbling}; hence, we present results based on small-body imagery as a case study. In principle, however, the proposed technique is also applicable to other irregularly-shaped objects, such as artificial spacecraft.

\subsection{Pole Estimation During the Approach Phase}

For missions to unknown or poorly known targets, it is often of interest to obtain a pole estimate during the target-approach phase, as the spacecraft nears its destination. Knowledge of the pole direction enables the definition of a rotating, target-fixed reference frame supporting other navigation-and-characterization tasks, such as shape reconstruction and surface-relative navigation\cite{palmer2022practical}. During approach, the target body often appears in onboard imagery as a low-to-medium resolution object---spanning from less than one pixel to hundreds of pixels---for an extended period of time. Once the target is resolved, the \textit{silhouette} of its irregular shape becomes visible in imagery. We define an object's silhouette as the occupancy mask (a binary image) of the target object in front of the imaged background. Unlike surface landmarks, which usually become distinguishable at higher resolutions, silhouettes can be extracted from both low- and high-resolution imagery.

\subsection{Related Work}

A variety of image-based methods have been proposed and effectively employed for small-body pole estimation. These can be subdivided into three categories, based on the image resolution of interest: (1) lightcurve-based methods, for unresolved-object imagery, (2) silhouette-based methods, for low- and medium-resolution images of the object, and (3) landmark- or pattern-based methods, for high-resolution images of the object.

\subsubsection{Pole Estimation for Unresolved Objects}
Lightcurve analysis has been extensively used to estimate the pole, shape, and rotation periods of small celestial bodies, typically through ground-based observations\cite{kaasalainen2001optimization,chng2022globally}. The lightcurve is the evolution of total brightness for an imaged object as it is observed over time. The key principle is that lightcurve signatures depend on the shape and rotational motion of the associated object. Such techniques are typically designed for unresolved-object observations, where the irregular shape is not directly observable, and require modeling surface-reflectance properties to constrain the optimization problem. This approach can lead to multiple shape and pole hypotheses, which cannot be disambiguated with limited data or unfavorable observation geometries. Furthermore, lightcurve-based estimates are sensitive to surface mismodeling, e.g., due to heterogeneities in albedo and reflectance properties.

\subsubsection{Pole Estimation for Low and Medium-resolution Objects}

Previous work investigated the use of silhouettes for pole estimation, suitable for lower-resolution imagery. Bandyopadhyay et al. propose a Pole-from-Silhouette technique where multiple pole-orientation hypotheses are evaluated through a grid search by matching predicted and observed silhouettes\cite{bandyopadhyay2021light}. For each pole hypothesis, and assuming knowledge of the rotation rate, silhouette predictions are computed by firstly reconstructing the 3D visual hull of the target object, using a Shape-from-Silhouette method, and then reprojecting the visual hull's silhouette onto the camera plane. This technique has been demonstrated as part of an autonomous-navigation pipeline for small-body exploration\cite{nesnas2021autonomous}. However, said approach is computationally expensive as it requires 3D-shape reconstruction for each pole hypothesis through the grid search, and relies on the assumption that the object's center-of-mass location relative to the camera is known.

\subsubsection{Pole Estimation for High-resolution Objects}

For sufficiently resolved target images, the pole orientation can be estimated by tracking surface landmarks across multiple images. The underlying principle here is that landmark tracks projected onto the camera plane depend on the orientation between the camera reference frame and the pole that landmarks rotate with respect to. Computing landmark tracks requires detecting and matching the same set of landmarks across multiple images, which can be challenging when surface lighting conditions and camera poses evolve across observations. For small-body missions, state-of-the-practice techniques rely on estimating surface-topography models, typically using a method known as Stereophotoclinometry (SPC)\cite{palmer2022practical,adam2023stereophotoclinometry}, which are then used to compute the appearance of landmarks and perform landmark tracking. To date, techniques such as SPC rely on complex ground-based operations and are not suitable for autonomous characterization.

Autonomous pole-estimation approaches have been proposed, such as the use of visual-feature-tracking algorithms---e.g., SIFT \cite{lowe2004distinctive}---for model-free landmark tracking\cite{panicucci2023vision,villa2020optical}. One limitation is that feature tracks are subject to drift when tracked across multiple images, especially for objects characterized by challenging lighting conditions and irregular surface topography\cite{morrell2020autonomous}, which can reduce the accuracy of the associated pole estimates. Further, tracking surface features requires high-resolution images, which are typically unavailable until the spacecraft is in close proximity to the body. In addition to visual features, \textit{circle-of-latitude} patterns, i.e., elliptical ``streaks" produced by stacking consecutive images over time, have been proposed as image patterns to estimate the pole orientation\cite{kuppa2024initial,christian2024pole}. With such methods, pole estimates can be sensitive to the quality of the extracted circle-of-latitude patterns, which in turn depend on image-alignment errors and surface appearance.

\subsection{Proposed Approach}

In this work, we present an algorithm to estimate the pole orientation by stacking silhouettes of the target object; we name this approach PoleStack. By imaging the silhouette evolution over time, as the target object rotates about its pole, the object's pole can be determined. As such, silhouettes allow to extract an early pole estimate during the long-lasting approach phase to determine the target's rotational motion before arrival. The proposed method is suitable for low-resolution imagery, challenging lighting and shadowing conditions, and common image-alignment errors (e.g., resulting from using the center of brightness). Such conditions can be encountered during approach as well as other mission phases.

This work is subdivided as follows: the problem formulation is presented in Section \ref{sec:problem_formulation}; the theoretical background supporting the proposed approach is described in Section \ref{sec:theor_dev}; the PoleStack algorithm is presented in Section \ref{sec:algorithm_overview}; experimental results showcasing PoleStack performance are reported in Section \ref{sec:results}; lastly, the work's conclusions are reported in Section \ref{sec:conclusions}. While a thorough theoretical discussion is provided, the actual algorithm consists of a few steps (see Algorithm \ref{alg:in-plane}).

\section{Problem Formulation}
\label{sec:problem_formulation}

\subsection{Preliminaries}

Consider an irregular object with surface $\Omega\subset \mathbb{R}^3$ rotating about a pole direction $\boldsymbol{\omega}$, where $\|\boldsymbol{\omega}\|=1$. Given a vector $\mathbf{v}$ and a reference frame $\mathcal{A}$, the notation $\mathbf{v}_\mathcal{A}$ indicates $\mathbf{v}$ expressed with respect to $\mathcal{A}$. Let $\mathcal{N}$ denote the inertial reference frame and $\mathcal{B}$ denote the body-fixed reference frame attached to the surface $\Omega$. The z-axis of $\mathcal{B}$ is parallel to $\boldsymbol{\omega}$, such that $\boldsymbol{\omega}_\mathcal{B}=[0,0,1]^\top$; the symbol $\top$ denotes the transposition operator. The origin of both $\mathcal{N}$ and $\mathcal{B}$ coincides with the object's center of mass. In this work, 3D vectors are expressed with respect to the body-fixed frame $\mathcal{B}$, unless otherwise specified, in which case the reference-frame notation is often omitted for brevity (e.g., we write $\mathbf{v}$ instead of $\mathbf{v}_\mathcal{B}$). 






Suppose that an observing camera acquires images of the surface $\Omega$. The camera pose is defined by the position $\mathbf{r} \in \mathbb{R}^3$ and the rotation matrix $\left[ \mathcal{B} \mathcal{C} \right] \in \mathbb{R}^{3\times 3}$ which transforms a vector from the camera-fixed reference frame $\mathcal{C}$ to the body-fixed frame $\mathcal{B}$, i.e., $\mathbf{v}_\mathcal{B} = \left[ \mathcal{B} \mathcal{C} \right] \mathbf{v}_\mathcal{C}$. In this work, we follow the OpenCV convention \cite{opencv_library} to define the camera frame $\mathcal{C}$, where the camera x-axis and y-axis are oriented from left to right and from top to bottom with respect to the image, whereas the z-axis points toward the observed scene.

Assuming the pinhole camera model \cite{hartley2003multiple}, a 3D surface point $\mathbf{p}=[p_1,p_2,p_3]^\top\in\Omega$ is observed in the image as a projected point $\mathbf{u}_\mathbf{p}=[u_\mathbf{p},v_\mathbf{p}]^\top\in\mathbb{P}^2$\footnote{$\mathbb{P}^2$ denotes the 2D projective space\cite{henry2023absolute}.}, where $[u_\mathbf{p},v_\mathbf{p}]^\top$ are the image coordinates expressed in units of pixels, given by:

\begin{equation}
    \label{eq:u_Cp}
    \bar{\mathbf{u}}_\mathbf{p} = C \bar{\mathbf{p}}
\end{equation}

where $\bar{\mathbf{p}}$ and $\bar{\mathbf{u}}_\mathbf{p}$ are $\mathbf{p}$ and $\mathbf{u}_\mathbf{p}$ expressed in homogeneous coordinates, respectively. $C\in \mathbb{R}^{3\times 4}$ is the camera-projection matrix, defined as: 

\begin{equation}
\label{eq:C_nu}
    C = K \bigl[ \left[ \mathcal{C} \mathcal{B} \right] \mid -\mathbf{r}_\mathcal{C} \bigr]
\end{equation}

where $K\in \mathbb{R}^{3\times 3}$ is the camera intrinsic matrix containing the calibration parameters. $\bigl[ \left[ \mathcal{C} \mathcal{B} \right] \mid -\mathbf{r}_\mathcal{C} \bigr] \in \mathbb{R}^{3\times 4}$ is known as the camera-extrinsic matrix, where $\left[ \mathcal{C} \mathcal{B} \right] = \left[ \mathcal{B} \mathcal{C} \right]^\top$ is the rotation matrix transforming a vector from $\mathcal{B}$ to $\mathcal{C}$ and $-\mathbf{r}_\mathcal{C}$ is the location of the object's center of mass with respect to the camera position, expressed in $\mathcal{C}$.

\subsection{Hovering-Camera Model}
\label{sec:hovering_camera_model}

For our purposes, it is convenient to express the body-fixed camera position $\mathbf{r}_\mathcal{B}$ in terms of spherical coordinates:

\begin{equation}
    \mathbf{r}_\mathcal{B} = r \begin{bmatrix}
        \mathrm{cos}(\lambda)\mathrm{cos}(\phi)\\
        \mathrm{cos}(\lambda)\mathrm{sin}(\phi)\\
        \mathrm{sin}(\lambda)
    \end{bmatrix}
\end{equation}

where $r = \|\mathbf{r}\|$ is the camera distance, $\lambda \in [-\frac{\pi}{2},\frac{\pi}{2}]$ is the camera latitude, and $\phi \in [0,2\pi)$ is the camera longitude. Then, we define a set of \textit{hovering camera views}, $\mathcal{V}_j$, as:

\begin{equation}
    \mathcal{V}_j = \{ \mathbf{r}(r,\lambda,\phi) \; | \; \phi \in [\phi_0,\phi_f],\, r = r_j,\, \lambda = \lambda_j \}
\end{equation}

where $[\phi_0,\phi_f]$ is a camera-longitude interval whereas $r_j$ and $\lambda_j$ are constant radius and latitude values, respectively. That is, $\mathcal{V}_j$ represents a set of camera positions located at a fixed distance and latitude but varying longitude with respect to the body-fixed frame. This condition is encountered when the camera inertial position is fixed---a hovering state---and the body rotates about its pole.

This model is particularly relevant for scenarios where the evolution of the surface appearance is driven by the target-object rotation. Such conditions commonly arise during spacecraft approach or hover phases relative to the target. During approach, variations in camera distance and latitude typically progress more slowly than the object's rotational motion, which drives changes in camera longitude. In practice, small variations in the object's apparent size can be compensated for using prior knowledge of the camera trajectory, supporting the ``hovering" assumption. Additionally, when using long-range observations, perspective effects of the imaged surface are negligible.

\subsection{Problem Statement}

In this work, we address the following problem. An irregular object with surface $\Omega$ rotating about its pole $\boldsymbol{\omega}$ is given. Consider a hovering-camera view set $\mathcal{V}_j = \{ \mathbf{r}(\phi_{j,1}), \dots, \mathbf{r}(\phi_{j,M_j}) \}$ and the associated image set $\mathcal{I}_j = \{ I_{j,1}, \dots, I_{j,M_j} \}$, where $I_{j,k}\in\mathbb{R}^{N\times N}$ is an image of the surface $\Omega$ collected from the camera position $\mathbf{r}(\phi_{j,k})$\footnote{For simplicity and without loss of generality, we carry out the discussion for square ($N\times N$) images, instead of rectangular ($M\times N$) images.}. Then, the objective is to estimate the inertial pole direction $\boldsymbol{\omega}_\mathcal{N}$ using image sets $\{ \mathcal{I}_1, \dots, \mathcal{I}_\aleph \}$ collected from the corresponding camera-view sets, $\{ \mathcal{V}_1, \dots, \mathcal{V}_\aleph \}$\footnote{In this work, we formulate the problem as estimating $\boldsymbol{\omega}$ with respect to the inertial frame, $\mathcal{N}$, but it is easy to show that $\boldsymbol{\omega}$ could equivalently be estimated with respect to a camera reference frames as well.}.

The rationale for of using a hovering-camera model is that, under certain assumptions (see Section \ref{sec:assumptions}), the pole direction observed across the hovering-camera viewset $\mathcal{V}_j$ remains constant across images $\{ I_{j,1}, \dots, I_{j,M_j} \}$. The observed pole direction is then estimated by combining the information in $\{ I_{j,1}, \dots, I_{j,M_j} \}$ to extract the evolution of the object rotating about its pole. 

\subsection{Assumptions}
\label{sec:assumptions}

The proposed approach is based on the following assumptions:

\begin{enumerate}
    \item The camera orientation with respect to the inertial frame, defined by $\left[ \mathcal{N} \mathcal{C} \right]$, is known. (If another frame $\mathcal{A}$ is used to estimate the pole direction $\boldsymbol{\omega}_\mathcal{A}$, then the rotation $\left[ \mathcal{A} \mathcal{C} \right]$ is known, instead.)
    \item The imaged surface $\Omega$ is entirely contained within the camera field of view.
    \item The inertial camera attitude $\left[ \mathcal{N} \mathcal{C} \right]$ is constant throughout the corresponding hovering-camera view set, and is such that the camera-boresight axis points toward the center of mass of the rotating object. While this assumption allows to simplify the theoretical development (Section \ref{sec:theor_dev}), we will relax it later on and generalize results to the case where the target center-of-mass location is unknown a priori. In practice, the center-of-mass image location need not be known, as discussed in Section \ref{sec:imaging_error_sources} and demonstrated in Section \ref{sec:results}.
    

    
    \item The silhouette has been previously extracted from the image, e.g., using thresholding techniques.
    \item The camera is perfectly calibrated, i.e., the camera intrinsic matrix $K$ is known.
\end{enumerate}

Note that the accuracy of the hovering-camera approximation depends on the time interval spanned by $\{ I_{j,1}, \dots, I_{j,M_j} \}$ and is generally more accurate for smaller time frames.

\section{Theoretical Development}
\label{sec:theor_dev}

This section presents the theoretical foundations of the PoleStack algorithm. We provide a mathematical model of the silhouette-stack image obtained from a hovering-camera viewset, describe its symmetry properties, error sources affecting overall symmetry, and the use of the Discrete Fourier Transform (DFT) for robust symmetry detection.

\subsection{Summary of Results}
\label{sec:summary_of_results}
The key results leveraged by the proposed approach can be summarized as follows:

\begin{enumerate}
    \item The image obtained by stacking---also known as co-adding---silhouette images collected across some camera-longitude range exhibits some level of reflective symmetry with respect to the pole direction projected onto the camera plane (see Figures \ref{fig:silh_stack_360deg}-\ref{fig:silh_stack_360deg_centrd}). The projected-pole direction can then be estimated by finding the direction of maximum symmetry in the silhouette-stack image. Note that this estimate only provides the pole-direction component on the camera plane.
    \item We identify three key error sources affecting silhouette-stack symmetry: partial camera-longitude coverage, surface shadowing, and silhouette-alignment errors. These effects manifest as an additive error term applied on the silhouette-stack image.
    \item The amplitude spectrum of the silhouette-stack image, obtained through the Discrete Fourier Transform (DFT), preserves the reflective symmetry of the silhouette-stack image while exhibiting translation invariance. This translation invariance makes frequency-domain symmetry detection particularly valuable when the image location of the object's center of mass is not well known.
    \item A 3D pole-direction estimate, including the out-of-plane component, can be obtained by combining (``triangulating") multiple in-plane pole estimates computed from different hovering-camera viewsets.
\end{enumerate}

\subsection{Symmetry of Silhouette Stacks}
\label{sec:perfect_silh}

In this section, we show that the evolution of silhouette images observed from a hovering camera across some longitude range exhibits some level of symmetry about the pole direction. In the presented formalism, we model images and camera-longitude intervals as continuous quantities, i.e., we neglects discretization effects due to image quantization and the finite number of camera views. In practice, this model can be representative of scenarios with sufficient image resolution and overlap between consecutive images across longitude, as empirically shown in Section \ref{sec:results}.

\subsubsection{Silhouette Observation Model}
\label{sec:silhouette_obs_model}

Consider a set of hovering-camera views $\mathcal{V}_j$ defined according to Section \ref{sec:hovering_camera_model}. To begin, we describe a perfect silhouette by neglecting surface-shadowing effects corrupting its appearance.

\begin{definition}
   \label{def:Omega_v}
    Let $\Omega\subset \mathbb{R}^3$ be a surface, $\mathbf{p}\in\Omega$ a surface point, and $\mathbf{r}(\phi)$ the position of a hovering camera located at longitude $\phi$. We define the visible surface $\Omega_v(\phi) \subset \Omega$ as observed from $\mathbf{r}(\phi)$ as:

    \begin{equation}
        \Omega_v(\phi) = \{ \mathbf{p}\in\Omega \; | \; \forall t \in (0,1), \mathbf{r}(\phi)+t(\mathbf{p}-\mathbf{r}(\phi))\notin \Omega \}
    \end{equation}
\end{definition}

Definition \ref{def:Omega_v} implies that the set of visible surface points $\mathbf{p}\in\Omega_v(\phi)$ are such that the segment connecting $\mathbf{r}(\phi)$ with $\mathbf{p}$ does not does not intersect the surface $\Omega$, i.e., $\mathbf{p}$ is not physically occluded by other surface points, when observed from $\mathbf{r}(\phi)$.

\begin{definition}
    \label{def:s_v}
    Let $C(\phi)$ be the camera-projection matrix (Equation \ref{eq:C_nu}) associated with a hovering-camera view $\mathbf{r}(\phi)$ and let $\Omega_v(\phi)$ be the corresponding visible surface. We define the silhouette $\mathcal{S}(\phi) \subset \mathbb{P}^2$ of the surface $\Omega$, as seen from $\mathbf{r}(\phi)$, as:

    \begin{equation}
        \mathcal{S}(\phi) = \{ \mathbf{u} \in \mathbb{P}^2 \; | \; \bar{\mathbf{u}}=C(\phi)\bar{\mathbf{p}},\,\mathbf{p}\in\Omega_v(\phi) \}
    \end{equation}

    where $\bar{\mathbf{u}}$ and $\bar{\mathbf{p}}$ are the homogeneous coordinates of $\mathbf{u}$ and $\mathbf{p}$, as shown by Equation \ref{eq:u_Cp}.
\end{definition}

According to Definition $\ref{def:s_v}$, a silhouette represent the image region containing the visible surface, coinciding with the image foreground. In this work, we represent silhouette regions using indicator functions that denote their occupancy in the image.

\begin{definition}
\label{def:indicator}
Let $\mathcal{P}$ and $\mathcal{Q}$ be two sets, such that $\mathcal{P} \subseteq \mathcal{Q}$, and let $\mathbf{q}\in\mathcal{Q}$. The indicator function of $\mathcal{P}$, denoted as $\mathbf{1}_\mathcal{P}:\mathcal{Q}\rightarrow \mathbb{B}$, is defined as:
    
    \begin{equation}
        \mathbf{1}_\mathcal{P}(\mathbf{q}) = \begin{cases}
    1 & \text{if } \;\; \mathbf{q} \in \mathcal{P} \\
    0   & \text{otherwise }
  \end{cases}
    \end{equation}

    where $\mathbb{B}=\{0,1\}$ is the Boolean set.
\end{definition}

\begin{definition}
\label{def:1_S}
    Given a silhouette $\mathcal{S}(\phi)$ (Definition \ref{def:s_v}), we define its indicator function $\mathbf{1}_{\mathcal{S}}(u,v; \phi)$ with parameter $\phi$ as:

    \begin{equation}
        \mathbf{1}_\mathcal{S}(u,v; \phi) =
        \begin{cases}
            1 & \text{if} \;\; \mathbf{u}=[u,v]^\top \in \mathcal{S}(\phi)\\
            0 & \text{otherwise}
        \end{cases}
    \end{equation}

    according to Definition \ref{def:indicator}.
    
\end{definition}

To facilitate the analysis, we introduce an image reference frame centered at the object's center of mass and oriented along the projected-pole direction, as described below.

\begin{definition}
\label{def:omega_proj}
    Given a camera-projection matrix $C$, we define the corresponding pole projection in the image plane as the direction $\boldsymbol{\omega}_\mathrm{proj}\in\mathbb{P}^2,\,\| 
\boldsymbol{\omega}_\mathrm{proj} = 1 \|$, such that:

    \begin{equation}
        \bar{\boldsymbol{\omega}}'_\mathrm{proj} = C \bar{\boldsymbol{\omega}}
    \end{equation}

    where $\bar{\boldsymbol{\omega}}'_\mathrm{proj}$ and $\bar{\boldsymbol{\omega}}$ are the vectors $\boldsymbol{\omega}'_\mathrm{proj}$ and $\boldsymbol{\omega}$ expressed in homogeneous coordinates, respectively, and

    \begin{equation}
    \label{eq:omega_proj=omega'/norm_omega'}
        \boldsymbol{\omega}_\mathrm{proj} = \dfrac{\boldsymbol{\omega}'_\mathrm{proj}}{\|\boldsymbol{\omega}'_\mathrm{proj}\|}.
    \end{equation}
    
\end{definition}

Note that $\boldsymbol{\omega}_\mathrm{proj}$ is normalized (Equation \ref{eq:omega_proj=omega'/norm_omega'}) after the projection step to ensure unit norm.

\begin{definition}
\label{def:alpha=atan2}
    We define the pole-projection angle $\alpha$ as the angle describing the direction of $\boldsymbol{\omega}_\mathrm{proj} = [\omega_{\mathrm{proj},u}, \omega_{\mathrm{proj},v}]^\top$ in the image plane, with respect to the image vertical axis, as

    \begin{equation}
        \alpha = \mathrm{atan2}(-\omega_{\mathrm{proj},u},-\omega_{\mathrm{proj},v}).
    \end{equation}

\end{definition}

Observe that, from Definition \ref{def:alpha=atan2}, $\alpha=0$ when $\boldsymbol{\omega}_\mathrm{proj}$ points upwards in the image, i.e., along the $-v$ direction. 

\begin{definition}
    \label{def:u'}
    We define the pole-oriented image coordinates $\mathbf{u}'=[u',v']^\top$ as the coordinate set such that, for an image point $\mathbf{u}=[u,v]^\top$:

    \begin{equation}
        \begin{bmatrix}
            u \\ v
        \end{bmatrix} =
        \begin{bmatrix}
            \mathrm{cos}(\alpha) & \mathrm{sin}(\alpha) \\
            -\mathrm{sin}(\alpha) & \mathrm{cos}(\alpha)
        \end{bmatrix}
        \begin{bmatrix}
            u' \\ v'
        \end{bmatrix}
    \end{equation}
\end{definition}

The quantities introduced so far are illustrated in Figure \ref{fig:silh_scheme}.

\begin{figure}
    \centering
    \includegraphics[width=0.5\linewidth]{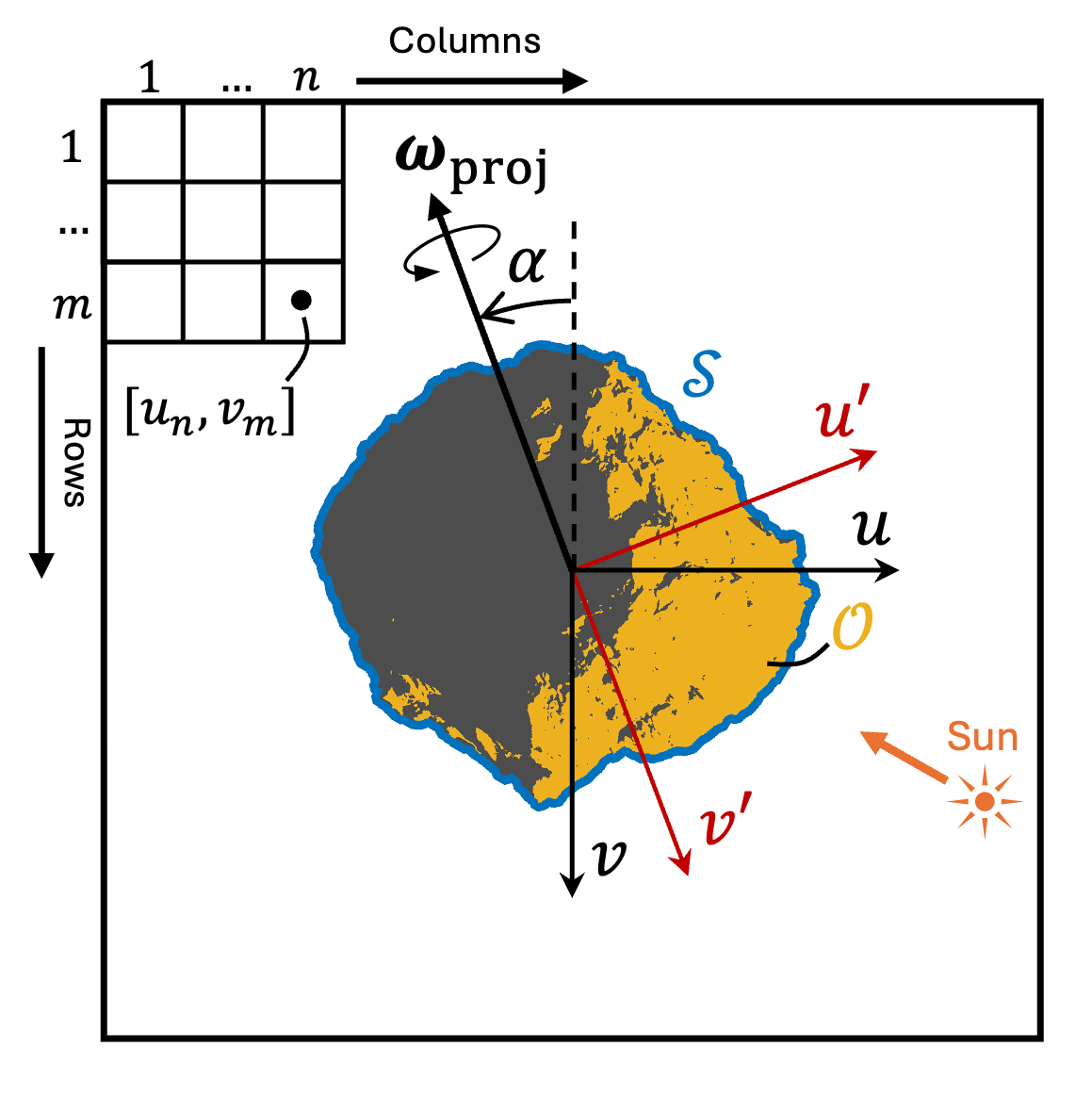}
    \caption{Schematic of an irregular object's silhouette observed in the image plane. Key parameters used in this work are reported: perfect silhouette ($\mathcal{S}$), observed silhouette ($\mathcal{O}$), projected-pole direction $\boldsymbol{\omega}_\mathrm{proj}$, pole-projection angle ($\alpha$), image coordinate system ($u$, $v$), pole-aligned coordinate system ($u'$, $v'$). A notional image array is also reported, where the center of the $mn$-th pixel ($m$-th row, $n$-th column) has image coordinates $[u_n, v_m]^\top$. The gray area corresponds to the shadowed silhouette region.}
    \label{fig:silh_scheme}
\end{figure}

We will now show that when viewed from a hovering camera, the trajectories of surface points $\mathbf{p}$ rotating about the pole $\boldsymbol{\omega}$ exhibit reflective symmetry with respect to the pole's projection $\boldsymbol{\omega}_\mathrm{proj}$.

\begin{definition}
\label{def:phi'}
    We define the camera-relative longitudinal coordinate $\phi'$ as
    
    \begin{equation}
        \phi' = \varphi - \phi
    \end{equation}

    where $\varphi$ is a generic longitude and $\phi$ is the camera longitude.
\end{definition}

\begin{definition}
    Given a visible surface point $\mathbf{p}\in\Omega_v$ with longitude coordinate $\phi_\mathbf{p}$, we define the camera-relative longitude of $\mathbf{p}$ as:

    \begin{equation}
        \phi'_\mathbf{p}=\phi_\mathbf{p}-\phi
    \end{equation}
\end{definition}

\begin{lemma}
    \label{lemma:s_v}

    Let $\mathbf{p}\in\Omega_v$ be a visible surface point and let $\mathbf{u}'_\mathbf{p}=[u'_\mathbf{p},v'_\mathbf{p}]^\top$ be its projection onto the image plane, expressed in pole-oriented coordinates. Then, the silhouette indicator function $\mathbf{1}_\mathcal{S}(u'_\mathbf{p},v'_\mathbf{p}; \phi')$ satisfies the following symmetry property:

    \begin{equation}
        \mathbf{1}_\mathcal{S}(-u'_\mathbf{p},v'_\mathbf{p}; -\phi'_\mathbf{p}) = \mathbf{1}_\mathcal{S}(u'_\mathbf{p},v'_\mathbf{p}; \phi'_\mathbf{p}) = 1, \, \phi'_\mathbf{p} = [-\pi,\pi]
    \end{equation}
\end{lemma}

\begin{proof}
    Let $\mathbf{p}(-\phi'_\mathbf{p})$ and $\mathbf{p}(\phi'_\mathbf{p})$ be the surface point $\mathbf{p}$ located at camera-relative longitudes $-\phi'_\mathbf{p}$ and $\phi'_\mathbf{p}$, respectively. By projecting such points onto the camera plane (Equation \ref{eq:u_Cp}), it is easy to show that

    \begin{equation}
    \label{eq:-uv=Cp,uv=Cp}
        [-\bar{u}'_\mathbf{p},\bar{v}'_\mathbf{p}]=C\bar{\mathbf{p}}(-\phi'_\mathbf{p}),\;\; [\bar{u}'_\mathbf{p},\bar{v}'_\mathbf{p}]=C\bar{\mathbf{p}}(\phi'_\mathbf{p})
    \end{equation}
    
    Combining Equation \ref{eq:-uv=Cp,uv=Cp} and Definition \ref{def:1_S}, the conclusion follows.
\end{proof}

Lemma \ref{lemma:s_v} articulates the root principle underlying our proposed pole-estimation technique: surface-point trajectories, when observed from the camera perspective, exhibit symmetry about the pole direction. In 3D space, a surface point $\mathbf{p}$ traces a circular arc centered around $\boldsymbol{\omega}$. The projection of this trajectory onto the image plane is typically observed as an elliptical arc\footnote{The projection of a circle onto a plane can theoretically yield a parabola or hyperbola\cite{christian2021tutorial}. However, these cases are more rarely encountered during scenarios of interested and are not discussed in this work.} whose minor axis aligns with $\boldsymbol{\omega}_\mathrm{proj}$. The image symmetry described in Lemma \ref{lemma:s_v} can be seen as a direct consequence of an ellipse's symmetry about its minor axis, which in our case coincides with $\boldsymbol{\omega}_\mathrm{proj}$, as illustrated in Figure \ref{fig:arc_symmetries}.

\begin{figure*}[t!]
    \centering
    \begin{subfigure}[t]{0.3\textwidth}
        \centering
        \includegraphics[width=\textwidth, keepaspectratio]{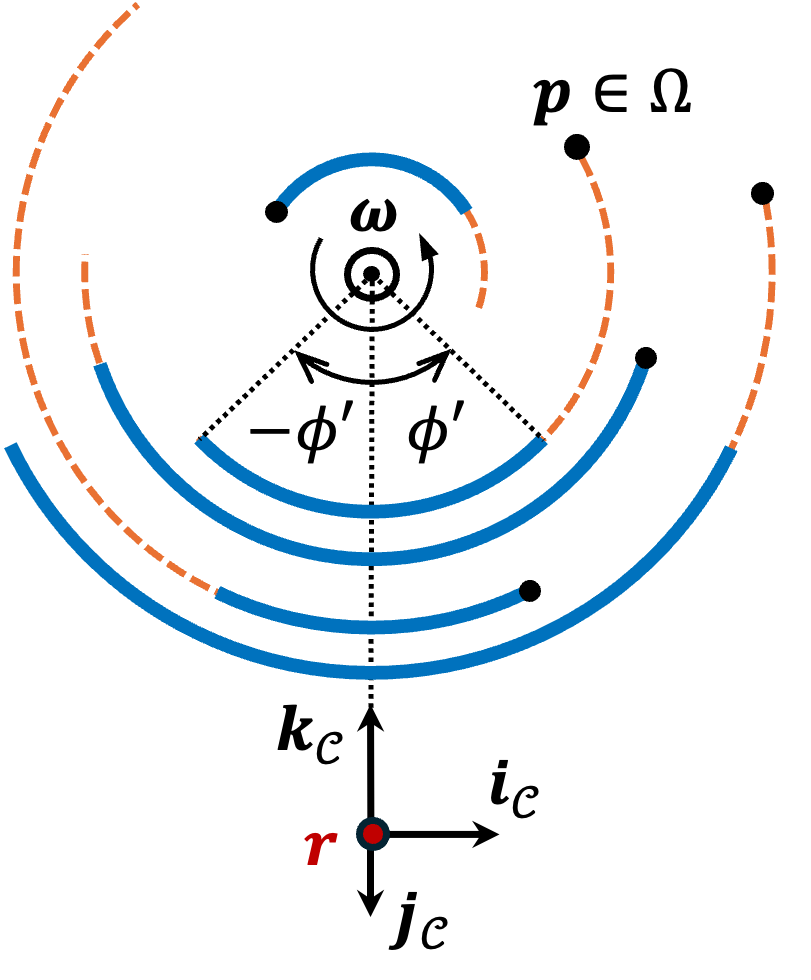}
        \caption{Top-down view. The pole direction $\boldsymbol{\omega}$ is oriented normal to the plane, pointing out of the page.}
    \end{subfigure}%
    ~ 
    \begin{subfigure}[t]{0.4\textwidth}
        \centering
        \includegraphics[width=\textwidth, keepaspectratio]{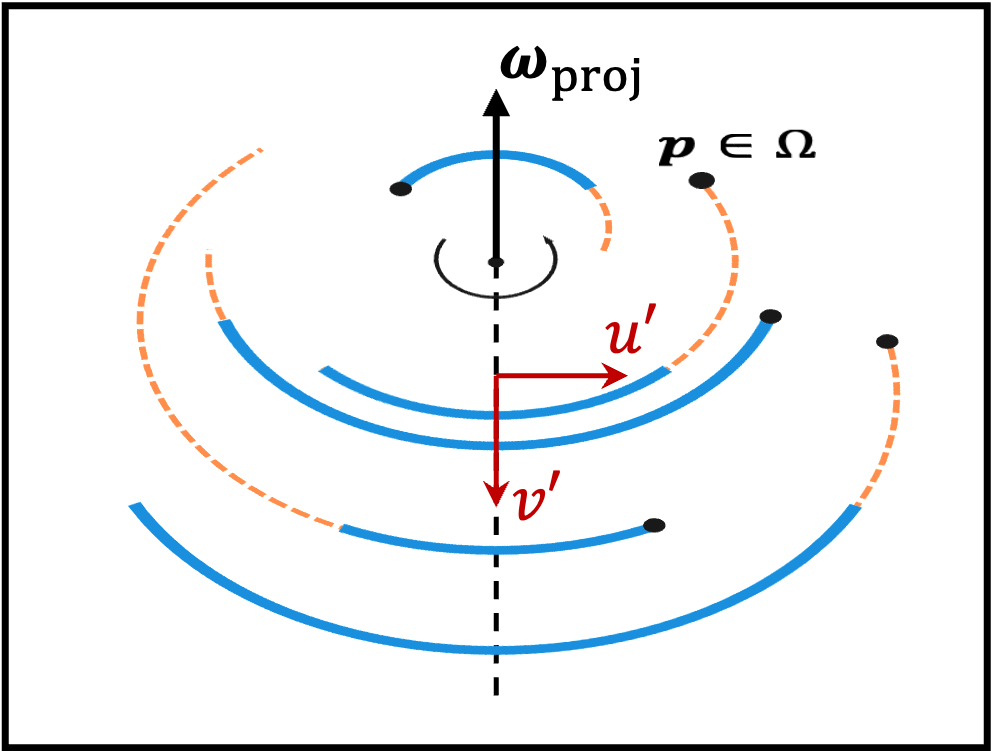}
        \caption{Camera view. Surface-point trajectories are observed as arcs of ellipse whose symmetric component exhibits reflective symmetry with respect to the projected pole $\boldsymbol{\omega}_\mathrm{proj}$.}
    \end{subfigure}
    \caption{Trajectory arcs traced by a subset of surface points $\mathbf{p}\in\Omega$ rotating about the pole $\boldsymbol{\omega}$, as observed from a hovering camera at position $\mathbf{r}$ with camera-frame axes $\mathbf{i}_\mathcal{C}, \mathbf{j}_\mathcal{C}, \mathbf{k}_\mathcal{C}$. Each arc begins at a distinct camera-relative longitudinal coordinate yet spans the same total longitude range. The symmetric (blue) and asymmetric (orange, dashed) components of each trajectory arc are shown as functions of the camera-relative longitude $\phi'$.}
    \label{fig:arc_symmetries}
\end{figure*}

\subsubsection{Silhouette-stack Model}
\label{sec:coadded_silh_model}

Section \ref{sec:silhouette_obs_model} describes the projective symmetry of surface-point trajectories. The proposed approach is based on processing silhouettes as a whole, rather than individual points. Hence, we now show that such a reflective symmetry is maintained when stacking silhouette observations across a camera-longitude range.

\begin{definition}
    \label{def:1_bar}
    Given an indicator function $\mathbf{1}_\mathcal{P}(\mathbf{q}; \zeta)$ (Definition \ref{def:indicator}) with parameter $\zeta \in \mathbb{R}$, we define the integral indicator function $\bar{\mathbf{1}}(\mathbf{q}; \zeta_0, \zeta_f):\mathcal{Q} \rightarrow \mathbb{R}$ in the interval $[\zeta_0,\zeta_f]$ as:

    \begin{equation}
        \bar{\mathbf{1}}(\mathbf{q}; \zeta_0, \zeta_f) = \int_{\zeta_0}^{\zeta_f} \mathbf{1}(\mathbf{q};\zeta) d\zeta.
    \end{equation}
    
\end{definition}

\begin{definition}
    \label{def:s_bar_v}
    Consider a set of silhouettes $\{\mathcal{S}(\phi) | \phi\in[\phi_0,\phi_f])\}$ and the corresponding indicator functions $\mathbf{1}_\mathcal{S}(u,v;\phi)$. The integral indicator function associated with the camera-longitude range $[\phi_0, \phi_f]$ is the function $\bar{\mathbf{1}}(u,v;\phi_0,\phi_f)$ given by

    \begin{equation}
    \label{eq:s_bar}
        \bar{\mathbf{1}}_\mathcal{S}(u,v;\phi_0,\phi_f) = \int_{\phi_0}^{\phi_f} \mathbf{1}_\mathcal{S}(u,v;\phi)\, d\phi.
    \end{equation}
\end{definition}

In this work, we use $\bar{\mathbf{1}}_\mathcal{S}(u,v;\phi_0,\phi_f)$ to represent the image intensity resulting from stacking silhouettes observed in $[\phi_0,\phi_f]$. Individual silhouettes are modeled by their indicator function (see Definition \ref{def:1_S}), and hence are represented as a Boolean occupancy mask in the image. The image intensity in $\bar{\mathbf{1}}_\mathcal{S}(u,v;\phi_0,\phi_f)$ then originates from stacking (i.e., co-adding) multiple Boolean masks $\mathbf{1}_\mathcal{S}(u,v;\phi)$. The decision to stack binary masks (i.e., silhouettes) rather than original images arises from the inherently geometric nature of surface-point trajectories, which binary representations capture more accurately. In contrast, variations in surface photometry carried by the original images can introduce effects which undermine this geometric symmetry.

Next, we demonstrate that the integral intensity function $\bar{\mathbf{1}}_\mathcal{S}$ exhibits reflective symmetry about the pole direction $\boldsymbol{\omega}_\mathrm{proj}$. Its symmetric component consists of surface-point trajectories that are mirrored about $\boldsymbol{\omega}_\mathrm{proj}$. As we will discuss below, the symmetric and asymmetric portions of a trajectory arc depend on the longitude of surface points relative to the observing camera.

\begin{definition}
    \label{def:s_signal}
     We define the symmetry-preserving silhouette region $\mathcal{S}_s(\phi) \subseteq \mathcal{S}(\phi)$ as:

    \begin{equation}
        \mathcal{S}_s(\phi) = \{ \mathbf{u}_\mathbf{p}(\phi'_\mathbf{p}),\mathbf{p}\in\Omega_v \; | \; \exists \; -\phi'_\mathbf{p},\; \phi\in [\phi_0,\phi_f]  \}
    \end{equation}

    Conversely, the symmetry-disrupting silhouette region $\mathcal{S}_n(\phi) \subseteq \mathcal{S}(\phi)$ is defined as:

    \begin{equation}
        \mathcal{S}_n(\phi) = \{ \mathbf{u}_\mathbf{p}(\phi'_\mathbf{p}) \; | \; \mathbf{u}_\mathbf{p}(\phi'_\mathbf{p}) \notin \mathcal{S}_s(\phi), \, \mathbf{p}\in\Omega_v \}
    \end{equation}

    where $\mathcal{S}_s(\phi)\cap \mathcal{S}_n(\phi)=\emptyset$. $\mathcal{S}_s$ and $\mathcal{S}_n$ are described by their respective indicator functions $\mathbf{1}_{\mathcal{S}_s}$ and $\mathbf{1}_{\mathcal{S}_n}$.
\end{definition}

It should be noted that the term ``symmetry" in Definition \ref{def:s_signal} refers to the \textit{evolution} of surface points across $\mathcal{S}(\phi)$, rather than the \textit{instantaneous} appearance of a single silhouette $\mathcal{S}(\phi)$. For any instantaneous silhouette, surface points belonging to the symmetry-preserving region $\mathcal{S}_s(\phi)$ are those observed from two symmetric camera-relative longitudes, $\phi'_\mathbf{p}$ and $-\phi'_\mathbf{p}$, as these result in symmetric silhouette regions with respect to $\boldsymbol{\omega}_\mathrm{proj}$. This effect can also be described in terms of individual surface-point trajectories, observed in the image plane as elliptical arcs. For any given surface point, its arc portion symmetric about $\boldsymbol{\omega}_\mathrm{proj}$ contributes to symmetry-preserving silhouettes, whereas the asymmetric portion contributes to symmetry-disrupting silhouette regions. Observe that Definition \ref{def:s_signal} is also valid when surface points are occluded by other points in any of the corresponding silhouettes, since occluding points lead to the same image coordinates, thereby preserving overall silhouette symmetry and lack thereof. The distinction between symmetry-preserving and symmetry-disrupting surface-point trajectories is illustrated in Figure \ref{fig:arc_symmetries}.

The symmetry-preserving silhouette $\mathcal{S}_s(\phi)$ provides information about the pole direction $\boldsymbol{\omega}_\mathrm{proj}$; hence, we treat $\mathcal{S}_s(\phi)$ as the signal extracted from the silhouette stack. Conversely, the symmetry-disrupting component $\mathcal{S}_n(\phi)$ decreases the overall level of symmetry about the pole, and hence is treated as the ``noise" component.

\begin{lemma}
\label{lemma:symm_+_asym_=1}
    \begin{equation}
        \mathbf{1}_{\mathcal{S}_s}(u,v;\phi) + \mathbf{1}_{\mathcal{S}_n}(u,v;\phi) = 1,\; \forall u,v,\phi
    \end{equation}
\end{lemma}

\begin{proof}
    Lemma \ref{lemma:symm_+_asym_=1} follows directly from Definition \ref{def:s_signal} and Definition \ref{def:1_S}.
\end{proof}

\begin{lemma}
    \label{lemma:u_p_in_S_symm}
    If $\mathbf{u}_\mathbf{p}(\phi)\in\mathcal{S}_s(\phi)$, then:

    \begin{equation}
        \mathbf{1}_{\mathcal{S}_s}(-u'_\mathbf{p},v'_\mathbf{p}; -\phi'_\mathbf{p}) = \mathbf{1}_{\mathcal{S}_s}(u'_\mathbf{p},v'_\mathbf{p}; \phi'_\mathbf{p})=1.
    \end{equation}
\end{lemma}

\begin{proof}
    It can easily be shown that Lemma \ref{lemma:u_p_in_S_symm} follows from Lemma \ref{lemma:s_v} and Definition \ref{def:s_signal}.
\end{proof}

\begin{corollary}
    \label{cor:s_noise}
    If $[\phi_0,\phi_f]=[0,2\pi)$, then $\mathbf{1}_{\mathcal{S}_n}(u,v;\phi)=0,\,\forall u,v,\phi$.
    
\end{corollary}

\begin{corollary}
    \label{cor:s_noise_no_motion}
    If $\phi_0=\phi_f$, then $\mathbf{1}_{\mathcal{S}_s}(u,v;\phi)=0,\,\forall u,v,\phi$.
    
\end{corollary}

Corollary \ref{cor:s_noise} shows that, in the edge case where the camera performs a full revolution about the surface $\Omega$, the trajectory of each surface point is fully symmetric with respect to $\boldsymbol{\omega}_\mathrm{proj}$ and the symmetry-disrupting silhouette region is empty for all observations. Conversely, as stated in Corollary \ref{cor:s_noise_no_motion}, if a single silhouette is collected, no silhouette region provides information on the pole direction.

Building on these principles, we can now show that silhouette-stack images exhibit reflective symmetry with respect to the projected-pole direction.

\begin{theorem}
\label{th:s_bar}
    Consider a silhouette set $\{\mathcal{S}(\phi) | \phi\in[\phi_0,\phi_f])\}$ collected through the camera-longitude interval $[\phi_0,\phi_f]$. Then, the integral indicator function $\bar{\mathbf{1}}_\mathcal{S}(u,v;\phi_0,\phi_f)$ is such that

    \begin{equation}
        \bar{\mathbf{1}}_\mathcal{S}(u,v;\phi_0,\phi_f) = \bar{\mathbf{1}}_{\mathcal{S}_s}(u,v;\phi_0,\phi_f) + \bar{\mathbf{1}}_{\mathcal{S}_n}(u,v;\phi_0,\phi_f)
    \end{equation}

    and $\bar{\mathbf{1}}_{\mathcal{S}_s}$ exhibits reflective symmetry with respect to $\boldsymbol{\omega}_\mathrm{proj}$, i.e.:

    \begin{equation}
    \label{eq:bar_1_S_symm}
        \bar{\mathbf{1}}_{\mathcal{S}_s}(-u',v';\phi_0,\phi_f) = \bar{\mathbf{1}}_{\mathcal{S}_s}(u',v';\phi_0,\phi_f)
    \end{equation}

\end{theorem}

\begin{proof}
    From Definition \ref{def:s_signal} and Lemma \ref{lemma:symm_+_asym_=1}, Equation \ref{eq:s_bar} can be expanded as:

    \begin{align}
        \bar{\mathbf{1}}_\mathcal{S}(u,v;\phi_0,\phi_f) &= \int_{\phi_0}^{\phi_f} \left( \mathbf{1}_{\mathcal{S}_s}(u,v;\phi) + \mathbf{1}_{\mathcal{S}_n}(u,v;\phi) \right)\, d\phi \\
        &= \int_{\phi_0}^{\phi_f} \mathbf{1}_{\mathcal{S}_s}(u,v;\phi)\, d\phi + \int_{\phi_0}^{\phi_f} \mathbf{1}_{\mathcal{S}_n}(u,v;\phi) \, d\phi \\
        &= \bar{\mathbf{1}}_{\mathcal{S}_s}(u,v;\phi_0,\phi_f) + \bar{\mathbf{1}}_{\mathcal{S}_n}(u,v;\phi_0,\phi_f)
    \end{align}
    
    where $\bar{\mathbf{1}}_{\mathcal{S}_s}$ and $\bar{\mathbf{1}}_{\mathcal{S}_n}$ are the integral indicator functions (Definition \ref{def:1_bar}) corresponding to $\mathbf{1}_{\mathcal{S}_s}$ and $\mathbf{1}_{\mathcal{S}_n}$, respectively. 
    
    Next, we study the symmetry of $\bar{\mathbf{1}}_{\mathcal{S}_s}$. Evaluating the left-hand side of Equation \ref{eq:bar_1_S_symm}:

    \begin{equation}
    \label{eq:bar_1_S_symm_v2}
        \bar{\mathbf{1}}_{\mathcal{S}_s}(-u',v';\phi_0,\phi_f) = \int_{\phi_0}^{\phi_f} \mathbf{1}_{\mathcal{S}_s}(-u',v';\phi)\, d\phi
    \end{equation}

    Definition \ref{def:s_signal} implies that:

    \begin{equation}
        \forall \phi\in[\phi_0,\phi_f] \; \exists \tilde{\phi}\in [\phi_0,\phi_f] \; | \; \mathbf{1}_{\mathcal{S}_s}(-u',v';\phi) = \mathbf{1}_{\mathcal{S}_s}(u',v';\tilde{\phi}) 
    \end{equation}

    Equation \ref{eq:bar_1_S_symm_v2} can then be rewritten as:

    \begin{equation}
        \bar{\mathbf{1}}_{\mathcal{S}_s}(-u',v';\phi_0,\phi_f) = \int_{\phi_0}^{\phi_f} \mathbf{1}_{\mathcal{S}_s}(u',v';\tilde{\phi})\, d\phi
    \end{equation}

    Performing a variable substitution $\phi \rightarrow \tilde{\phi}$, it is easy to show that $d\tilde{\phi} = -d\phi$ and that the limits of integration become $\phi_f$ (lower) and $\phi_0$ (upper), respectively. Hence, we have:

    \begin{align}
        \bar{\mathbf{1}}_{\mathcal{S}_s}(-u',v';\phi_0,\phi_f) &= - \int_{\phi_f}^{\phi_0} \mathbf{1}_{\mathcal{S}_s}(u',v';\tilde{\phi})\, d\tilde{\phi}\\
        &= \int_{\phi_0}^{\phi_f} \mathbf{1}_{\mathcal{S}_s}(u',v';\tilde{\phi})\, d\tilde{\phi}\\
        &= \bar{\mathbf{1}}_{\mathcal{S}_s}(u',v';\phi_0,\phi_f)
    \end{align}

\end{proof}

Theorem \ref{th:s_bar} describes the fundamental principle leveraged by the proposed PoleStack algorithm: a silhouette-stack image exhibits some level of reflective symmetry with respect to the pole projection $\boldsymbol{\omega}_\mathrm{proj}$, and hence can be used for pole estimation. When the surface is imaged through a full rotation ($\phi \in [0,2\pi)$) and error sources are neglected, the silhouette-stack model is fully symmetric with respect to the pole projection, $\boldsymbol{\omega}_\mathrm{proj}$. However, silhouette stacks collected through partial camera-longitude arcs ($0 < \phi_0 < \phi_f < 2\pi$) also exhibit some level of symmetry which, intuitively, increases with the camera-longitude span. In Section \ref{sec:results}, we empirically show that a full camera revolution around the surface is not necessary for effective pole estimation in practice. Figures \ref{fig:silh_stack_360deg} and \ref{fig:silh_stack_90deg} show examples of silhouette-stack images, and the corresponding symmetric and asymmetric components, for an irregular body observed across a full rotation and a partial rotation, respectively.

\begin{figure*}[t!]
    \centering
    \begin{subfigure}[t]{0.33\textwidth}
        \centering
        \includegraphics[width=\textwidth, keepaspectratio]{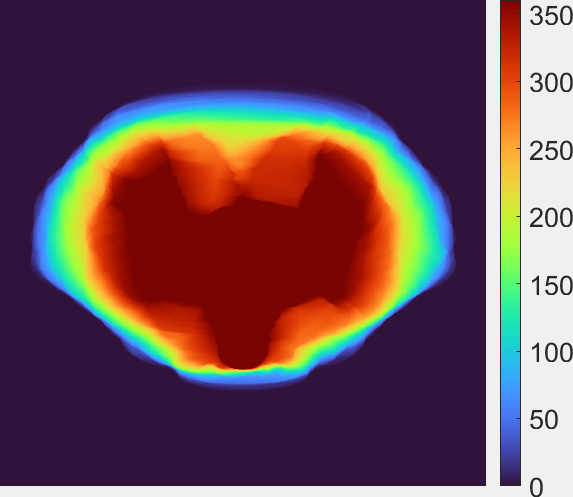}
        \caption{Silhouette-stack Image ($\bar{\mathbf{1}}_\mathcal{S}$)}
    \end{subfigure}%
    ~
    \begin{subfigure}[t]{0.33\textwidth}
        \centering
        \includegraphics[width=\textwidth, keepaspectratio]{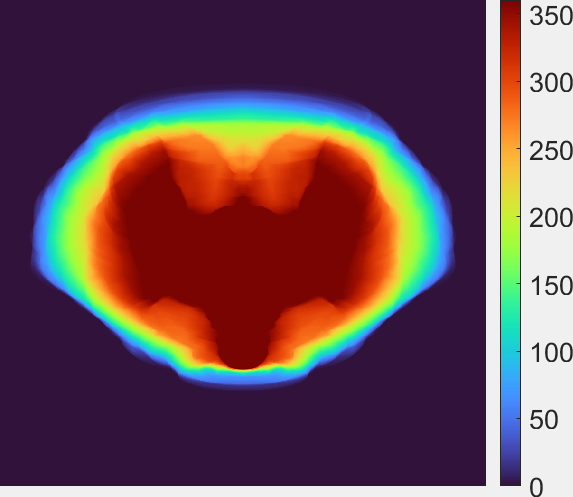}
        \caption{Symmetric Component ($\bar{\mathbf{1}}_{\mathcal{S}_s}$)}
    \end{subfigure}%
    ~
    \begin{subfigure}[t]{0.33\textwidth}
        \centering
        \includegraphics[width=\textwidth, keepaspectratio]{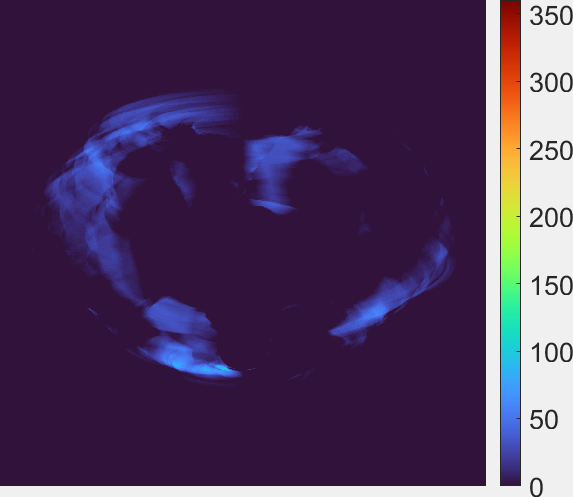}
        \caption{Asymmetric Component ($\bar{\mathbf{1}}_{\mathcal{S}_n}$)}
    \end{subfigure}
    
    \begin{subfigure}[t]{0.33\textwidth}
        \centering
        \includegraphics[width=\textwidth, keepaspectratio]{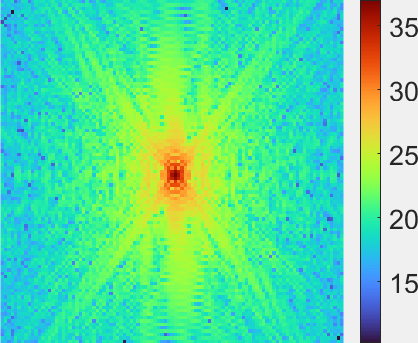}
        \caption{Log-power amplitude spectrum ($\mathrm{log}(1+\bar{\mathbf{1}}^2_\mathcal{S})$, see Section \ref{sec:fourier}, \ref{sec:algorithm_overview})}
    \end{subfigure}%
    ~
    \begin{subfigure}[t]{0.33\textwidth}
        \centering
        \includegraphics[width=\textwidth, keepaspectratio]{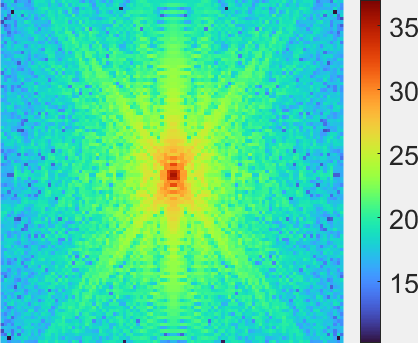}
        \caption{Symmetric Component}
    \end{subfigure}%
    ~
    \begin{subfigure}[t]{0.33\textwidth}
        \centering
        \includegraphics[width=\textwidth, keepaspectratio]{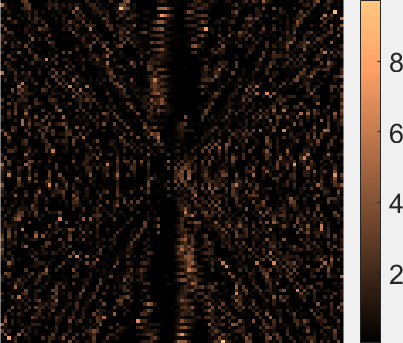}
        \caption{Asymmetric Component}
    \end{subfigure}
    \caption{Silhouette-stack image of comet 67P-C/G observed across a full rotation ($\phi=[0,2\pi]$), with a $1^\circ$ longitude interval between consecutive images, in the spatial domain (top) and frequency domain (bottom). Note that the image lacks perfect symmetry due to using a finite number of observations.}
    \label{fig:silh_stack_360deg}
\end{figure*}

\begin{figure*}[t!]
    \centering
    \begin{subfigure}[t]{0.33\textwidth}
        \centering
        \includegraphics[width=\textwidth, keepaspectratio]{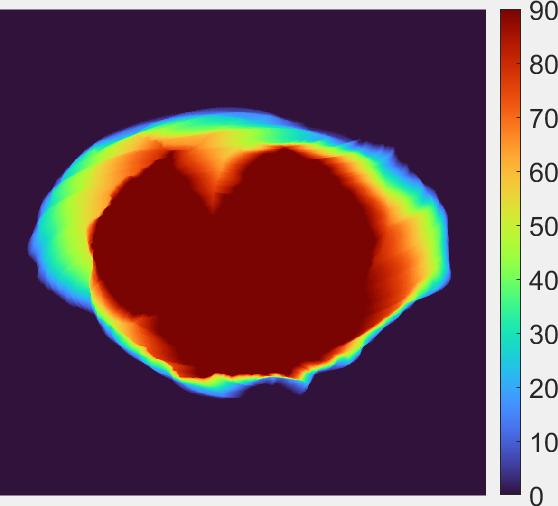}
        \caption{Silhouette-stack Image ($\bar{\mathbf{1}}_\mathcal{S}$)}
    \end{subfigure}%
    ~
    \begin{subfigure}[t]{0.33\textwidth}
        \centering
        \includegraphics[width=\textwidth, keepaspectratio]{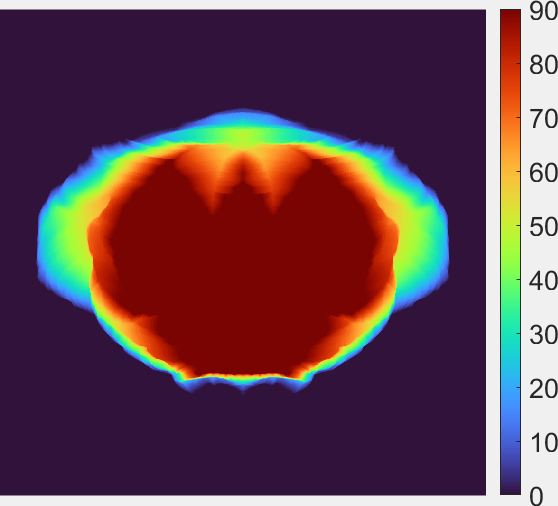}
        \caption{Symmetric Component ($\bar{\mathbf{1}}_{\mathcal{S}_s}$)}
    \end{subfigure}%
    ~
    \begin{subfigure}[t]{0.33\textwidth}
        \centering
        \includegraphics[width=\textwidth, keepaspectratio]{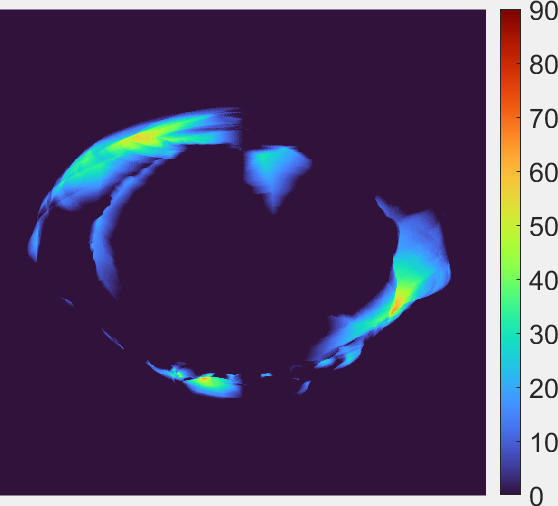}
        \caption{Asymmetric Component ($\bar{\mathbf{1}}_{\mathcal{S}_n}$)}
    \end{subfigure}
    
    \begin{subfigure}[t]{0.33\textwidth}
        \centering
        \includegraphics[width=\textwidth, keepaspectratio]{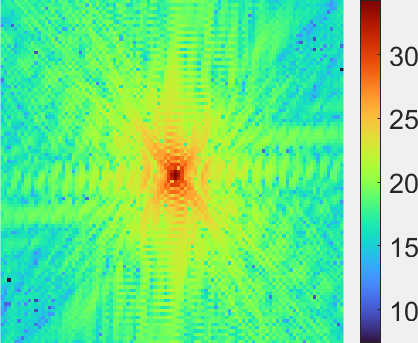}
        \caption{Log-power amplitude spectrum ($\mathrm{log}(1+\bar{\mathbf{1}}^2_\mathcal{S})$, see Sections \ref{sec:fourier}, \ref{sec:algorithm_overview})}
    \end{subfigure}%
    ~
    \begin{subfigure}[t]{0.33\textwidth}
        \centering
        \includegraphics[width=\textwidth, keepaspectratio]{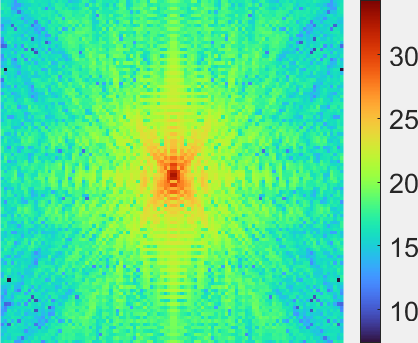}
        \caption{Symmetric Component}
    \end{subfigure}%
    ~
    \begin{subfigure}[t]{0.33\textwidth}
        \centering
        \includegraphics[width=\textwidth, keepaspectratio]{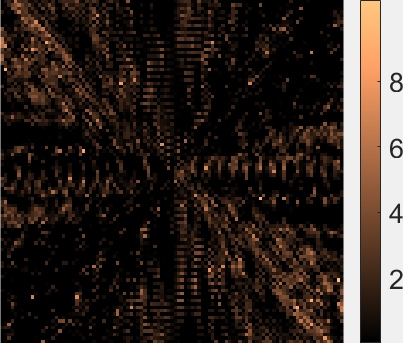}
        \caption{Asymmetric Component}
    \end{subfigure}
    \caption{Silhouette-stack image of comet 67P-C/G observed across a quarter of rotation ($\phi=[0,\pi/2]$), with a $1^\circ$ longitude interval between consecutive images, in the spatial domain (top) and frequency domain (bottom).}
    \label{fig:silh_stack_90deg}
\end{figure*}


Thus far, we developed a theoretical framework based on a perfect-silhouette model (Definition \ref{def:s_v}), neglecting error sources. Next, we relax this assumption and study common effects corrupting the symmetry of silhouette-stack images.

\subsection{Effect of Imaging Error Sources}
\label{sec:imaging_error_sources}


In practice, silhouette observations and the resulting silhouette stacks are corrupted by imaging error sources. In this work, we address two primary effects: surface shadowing and image-registration errors. Using a theoretical framework, we discuss the robustness of the proposed pole-estimation approach to these effects for practical applications.

\subsubsection{Surface Shadowing Effects}

The surface's irregular shape and topography create shadowing effects, where some regions cast shadows onto others, altering its observed appearance. Shadow patterns depend on the geometric relationship between camera, sun, and surface, and cannot be accurately predicted without a high-resolution surface model. The extent of shadowing typically increases with the sun phase angle\footnote{The sun phase angle is defined as the angle between the target-to-sun direction and the target-to-observer direction, as per Definition \ref{lemma:O_phi}}. When shadows extend to the limb, they can modify the silhouette shape, reducing the overall symmetry of silhouette stacks relative to the pole direction. In this section, we propose a model to account for surface shadowing and examine its impact on stacked silhouettes.





\begin{definition}
   \label{def:Omega_l}
    Let $\Omega\subset \mathbb{R}^3$ be a surface, $\mathbf{p}\in\Omega$ a surface point, and $\mathbf{r}_\mathrm{Sun}(\phi)$ the Sun position. We define the illuminated surface $\Omega_l(\phi) \subset \Omega$ as:

    \begin{equation}
        \Omega_l(\phi) = \{ \mathbf{p}\in\Omega \; | \; \forall t \in (0,1), \mathbf{r}_\mathrm{Sun}(\phi)+t(\mathbf{p}-\mathbf{r}_\mathrm{Sun}(\phi))\notin \Omega \}
    \end{equation}
\end{definition}




\begin{definition}
    \label{def:Omega_o}
    We define the observable surface $\Omega_o(\phi) \subset \Omega$ as the surface region which is both visible from the hovering camera $\mathbf{r}(\phi)$ and illuminated, i.e.:

    \begin{equation}
        \Omega_o(\phi) = \Omega_v(\phi) \cap \Omega_l(\phi)
    \end{equation}

\end{definition}

\begin{definition}
    \label{def:s_o}
    Given the observable surface $\Omega_o \subset \Omega$ and a hovering-camera view $\mathbf{r}(\phi)$, we define the observable silhouette $\mathcal{O}(\phi) \subset \mathbb{P}^2$ as:

    \begin{equation}
        \mathcal{O}(\phi) = \{ \mathbf{u} \in \mathbb{P}^2 \; | \; \bar{\mathbf{u}}=C(\phi)\bar{\mathbf{p}},\,\mathbf{p}\in\Omega_o(\phi) \}.
    \end{equation}

    $\mathcal{O}(\phi)$ is described by the indicator function $\mathbf{1}_\mathcal{O}(u,v;\phi)$.
\end{definition}

    






\begin{lemma}
\label{lemma:O_phi}
    Let $g=\mathrm{acos}\left( \dfrac{\mathbf{r}^\top_\mathrm{Sun}\mathbf{r}(\phi)}{\|\mathbf{r}_\mathrm{Sun}\|\|\mathbf{r}(\phi)\|} \right)\in [0,\pi]$ be the sun phase angle. Under the assumption of collimated incident sunlight, i.e., that the sun is infinitely far from the object, we have

    \begin{equation}
    \mathcal{O}(\phi) \begin{cases}
    =\mathcal{S}(\phi) & \text{if } \;\; g = 0 \\
    \subseteq \mathcal{S}(\phi)  & \text{otherwise }
  \end{cases}
    \end{equation}
\end{lemma}

\begin{proof}
    It is easy to show that Lemma \ref{lemma:O_phi} follows from the geometry between the surface, sun, and observer. 
\end{proof}

The region $\mathcal{O}(\phi)$ represents the silhouette region observed in imagery when surface-shadowing effects are present, and hence represents silhouette observations extracted from images of space objects in practice.

\begin{lemma}
\label{lemma:O_symm}
    Let $\mathcal{O}_s(\phi)$ and $\mathcal{O}_n(\phi)$ be the symmetry-preserving and symmetry-disrupting silhouette portions of $\mathcal{O}(\phi)$, according to Definition \ref{def:s_signal}. Further, let $\bar{\mathbf{1}}_\mathcal{O}(u,v;\phi_0,\phi_f)$ be the integral indicator function associated with $\mathbf{1}_\mathcal{O}(u,v;\phi)$, for $\phi \in [\phi_0,\phi_f]$ (Definition \ref{def:1_bar}). Then, we have:
    
    \begin{equation}
    \label{eq:1_O_bar_symm_+_asym}
        \bar{\mathbf{1}}_\mathcal{O}(u,v;\phi_0,\phi_f) = \bar{\mathbf{1}}_{\mathcal{O}_s}(u,v;\phi_0,\phi_f) + \bar{\mathbf{1}}_{\mathcal{O}_n}(u,v;\phi_0,\phi_f)
    \end{equation}

    where $\bar{\mathbf{1}}_{\mathcal{O}_s}$ and $\bar{\mathbf{1}}_{\mathcal{O}_n}$ are the symmetric and asymmetric components of $\bar{\mathbf{1}}_\mathcal{O}$, respectively, as described in Theorem \ref{th:s_bar}. Furthermore,

    \begin{equation}
    \label{eq:1_O_symm_bar_leq_1_S_symm_bar}
        \bar{\mathbf{1}}_{\mathcal{O}_s}(u,v;\phi_0,\phi_f) \leq \bar{\mathbf{1}}_{\mathcal{S}_s}(u,v;\phi_0,\phi_f),\; u,v \in \mathbb{P}^2,\, 0 \leq \phi_0 < \phi_f < 2\pi
    \end{equation}
\end{lemma}

\begin{proof}
    The proof for Equation \ref{eq:1_O_bar_symm_+_asym} follows an analogous approach to that presented for Theorem \ref{th:s_bar}. Equation \ref{eq:1_O_symm_bar_leq_1_S_symm_bar} follows from Lemma \ref{lemma:O_phi}, which implies

    \begin{equation}
        \mathbf{1}_\mathcal{O}(u,v;\phi) \leq \mathbf{1}_\mathcal{S}(u,v;\phi),\, \phi \in [0,2\pi).
    \end{equation}

\end{proof}

Equation \ref{eq:1_O_symm_bar_leq_1_S_symm_bar} mathematically confirms the intuitive understanding that surface shadows reduce silhouette stack symmetry relative to the pole. This effect suggests that silhouette observations should ideally be acquired at low sun phase angles to minimize shadowing. However, our empirical results (Section \ref{sec:results}) demonstrate that the proposed pole-estimation method remains robust even at high sun phases.

\subsubsection{Image Registration Errors}

The silhouette-stack model presented in Section \ref{sec:coadded_silh_model} assumes that the inertial camera attitude is fixed and pointing toward the object's center of mass (as assumed in Section \ref{sec:assumptions}). In practice, however, the center-of-mass image location is typically unknown a priori and can vary across observations due to camera-attitude dynamics. In these cases, an image-registration procedure is necessary to effectively align silhouette observations and produce the silhouette stack. In this section, we discuss the effect of image-alignment errors on silhouette-stack images, showing that silhouette stacks are inherently robust to common silhouette-alignment errors such as those originating from using the object's center of brightness for registration.

In general, the motion of the object's center of mass across the image plane results from changes in camera position and pointing relative to the center-of-mass location. For a hovering camera (as we assume in Section \ref{sec:assumptions}), the inertial camera position remains fixed, meaning the apparent center-of-mass motion stems solely from camera-attitude changes. When combining silhouette observations taken from different camera attitudes, these images require proper registration to generate an accurate silhouette-stack image.

The registration process can be accomplished through two distinct approaches. The first method transforms silhouettes to a common camera reference frame, assuming that the inertial attitude for each observation is known. The second method aligns silhouette \textit{centroids} by applying translation corrections in image space. In this context, a centroid represents a geometric quantity derived from image data that serves as a proxy for the object's center of mass when direct measurement is not possible. The brightness moment algorithm offers a standard approach for centroid extraction, calculating a weighted average of pixel locations based on their intensities within a specified image region\cite{owen2011methods}.

The choice between registration methods depends on camera attitude knowledge. The first approach requires accurate camera reference frame data, ideally at pixel or sub-pixel precision, making it optimal when reliable inertial attitude estimates are available. The second method, centroid-based alignment, proves more suitable when attitude uncertainties are significant. While this technique can mitigate registration errors along image-plane directions, it cannot correct for camera-twist attitude errors through centroid alignment alone. An important consideration is that the centroid location does not perfectly correspond to the actual center of mass, and their relative positions can vary between images based on specific pixel intensities. This misalignment introduces potential registration errors. Despite these limitations, our results in Section \ref{sec:results} demonstrate that centroid-based alignment remains an effective approach in practical applications.

In this work, silhouette-registration errors are modeled as an additive term, as formalized in Definition \ref{def:1_O_bar_tilde}. When registration is implemented effectively, these errors typically have less impact than the asymmetries introduced by partial surface coverage and surface shadowing discussed earlier. Small silhouette-registration errors typically affect the peripheral portion of the stacked silhouette, leaving most of the information and related symmetry intact. This effect is illustrated in Figure \ref{fig:silh_stack_360deg_centrd}.

\begin{definition}
\label{def:1_O_bar_tilde}
    Given the integral silhouette indicator function $\bar{\mathbf{1}}_\mathcal{O}(u,v;\phi_0,\phi_f)$, we define the integral indicator function of the aligned silhouettes, $\tilde{\bar{\mathbf{1}}}_\mathcal{O}$, as:

    \begin{equation}
        \tilde{\bar{\mathbf{1}}}_\mathcal{O}(u,v;\phi_0,\phi_f) = \bar{\mathbf{1}}_\mathcal{O}(u,v;\phi_0,\phi_f) + \nu_\mathrm{reg}(u,v;\phi_0,\phi_f)
    \end{equation}

    where $\nu_\mathrm{reg}: \mathbb{P}^2 \rightarrow \mathbb{R}$ is the function describing the image intensity originating from silhouette-registration errors.
    
\end{definition}

\begin{figure*}[t!]
    \centering
    \begin{subfigure}[t]{0.33\textwidth}
        \centering
        \includegraphics[width=\textwidth, keepaspectratio]{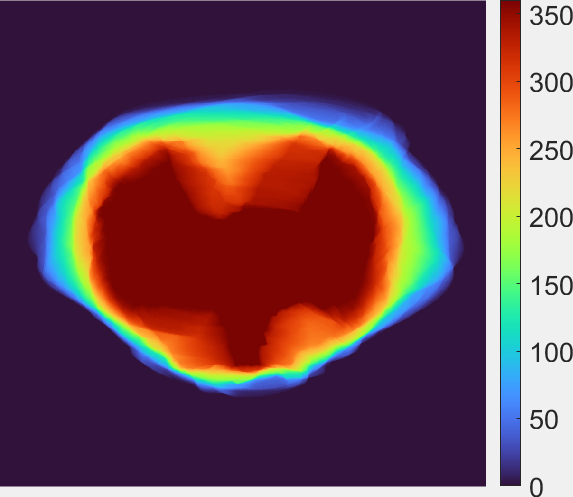}
        \caption{Silhouette-stack Image ($\tilde{\bar{\mathbf{1}}}_\mathcal{S}$)}
    \end{subfigure}%
    ~
    \begin{subfigure}[t]{0.33\textwidth}
        \centering
        \includegraphics[width=\textwidth, keepaspectratio]{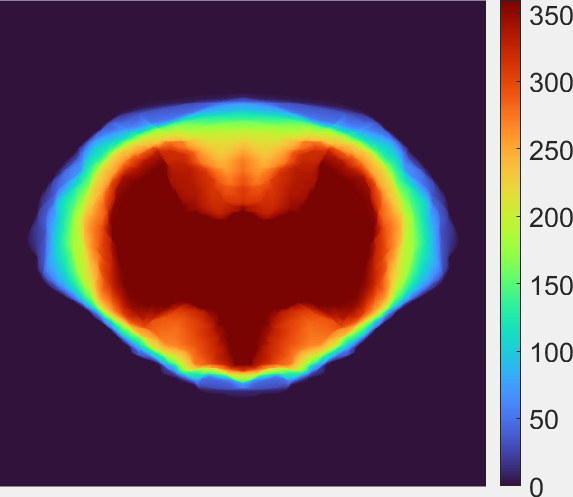}
        \caption{Symmetric Component ($\tilde{\bar{\mathbf{1}}}_{\mathcal{S}_s}$)}
    \end{subfigure}%
    ~
    \begin{subfigure}[t]{0.33\textwidth}
        \centering
        \includegraphics[width=\textwidth, keepaspectratio]{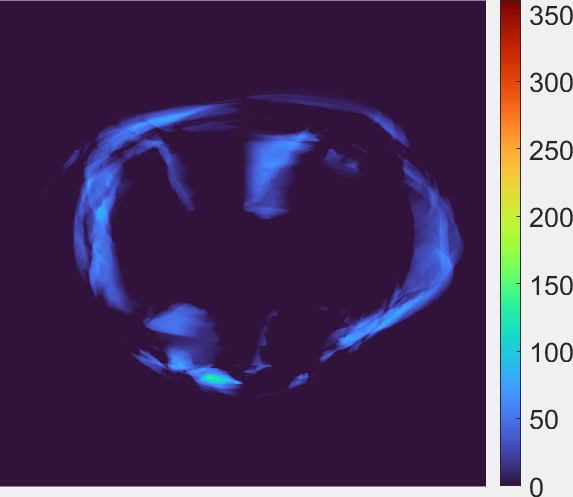}
        \caption{Asymmetric Component ($\tilde{\bar{\mathbf{1}}}_{\mathcal{S}_n}$)}
    \end{subfigure}
    
    \begin{subfigure}[t]{0.33\textwidth}
        \centering
        \includegraphics[width=\textwidth, keepaspectratio]{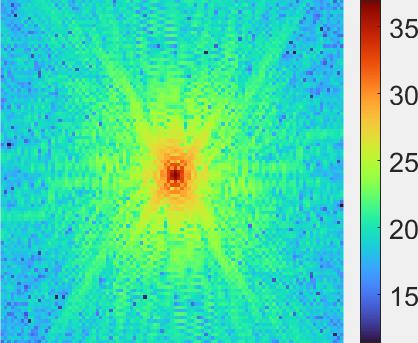}
        \caption{Log-power amplitude spectrum ($\mathrm{log}(1+\tilde{\bar{\mathbf{1}}}^2_\mathcal{S})$, see Section \ref{sec:fourier}, \ref{sec:algorithm_overview})}
    \end{subfigure}%
    ~
    \begin{subfigure}[t]{0.33\textwidth}
        \centering
        \includegraphics[width=\textwidth, keepaspectratio]{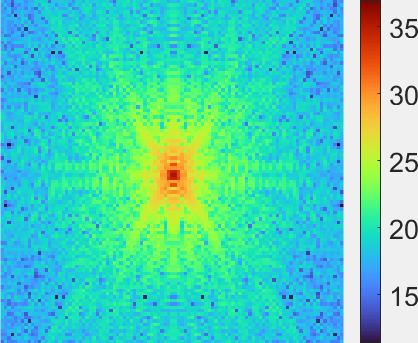}
        \caption{Symmetric Component}
    \end{subfigure}%
    ~
    \begin{subfigure}[t]{0.33\textwidth}
        \centering
        \includegraphics[width=\textwidth, keepaspectratio]{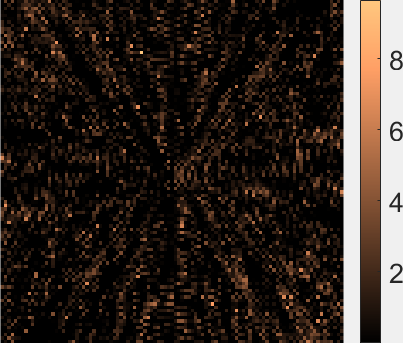}
        \caption{Asymmetric Component}
    \end{subfigure}
    \caption{Centroid-aligned silhouette-stack image of comet 67P-C/G observed across a full rotation ($\phi=[0,2\pi]$), with a $1^\circ$ longitude interval between consecutive images, in the spatial domain (top) and frequency domain (bottom). Individual silhouettes are registered by aligning their brightness centroid. Despite image-registration errors originating from this procedure, a high level of symmetry is preserved in the silhouette-stack image.}
    \label{fig:silh_stack_360deg_centrd}
\end{figure*}

\subsection{Silhouette-stack Image Model}

Building upon Sections \ref{sec:perfect_silh} and \ref{sec:imaging_error_sources}, we can derive a model for the silhouette-stack image that accounts for both surface shadowing and silhouette-registration errors. 

\begin{theorem}
\label{th:1_O=1_O_symm+e}
    Given a set of observed silhouettes $\{ \mathcal{O}(\phi), \phi \in [\phi_0,\phi_f] \}$, the integral silhouette indicator function $\tilde{\bar{\mathbf{1}}}_\mathcal{O}(u,v;\phi_0,\phi_f)$ can be written as:

    \begin{equation}
        \tilde{\bar{\mathbf{1}}}_\mathcal{O}(u,v;\phi_0,\phi_f) = \bar{\mathbf{1}}_{\mathcal{O}_s}(u,v;\phi_0,\phi_f) + e(u,v;\phi_0,\phi_f)
    \end{equation}

    where $\bar{\mathbf{1}}_{\mathcal{O}_s}$ is a symmetric function with respect to $\boldsymbol{\omega}_\mathrm{proj}$ and $e:\mathbb{P}^2 \rightarrow \mathbb{R}$ is the error function given by:

    \begin{equation}
        e(u,v;\phi_0,\phi_f) = \bar{\mathbf{1}}_{\mathcal{O}_n}(u,v;\phi_0,\phi_f) + \nu_\mathrm{reg}(u,v;\phi_0,\phi_f).
    \end{equation}
    
\end{theorem}

\begin{proof}
    Theorem \ref{th:1_O=1_O_symm+e} is easily obtained combining the image-alignment error model (Definition \ref{def:1_O_bar_tilde}) with the symmetry-preserving and symmetry-disrupting components of the integral silhouette indicator function (Lemma \ref{lemma:O_symm}).
\end{proof}

Theorem \ref{th:1_O=1_O_symm+e} establishes the mathematical foundation for pole estimation using silhouette stacks in practical scenarios by decomposing the silhouette-stack image into its symmetric and asymmetric components. The theorem demonstrates that signal and noise terms combine additively in the model, with their relative magnitudes determined by various error sources. The signal-to-noise ratio improves as the camera-longitude range increases, but degrades with stronger surface shadowing and larger image-alignment errors.

Before continuing the analysis, we must address an important distinction between theory and implementation. While our previous derivations treat image data and camera longitudes as continuous quantities, practical applications involve discrete images captured from a finite set of camera positions and longitudes. The following section establishes discrete representations for these quantities and their relationship to the continuous formulation previously developed.

\begin{definition}
    \label{def:O_img}
    Given an observed silhouette's indicator function, $\mathbf{1}_\mathcal{O}(\phi)$, the corresponding observed-silhouette digital image $O(m,n;\phi) \in \mathbb{B}^{N\times N}$ is such that
    
    \begin{equation}
        O(m,n;\phi) = \mathrm{sup}\left( \mathbf{1}_O(u,v),\, u_n - 0.5 \leq u < u_n + 0.5,\, v_m - 0.5 \leq v < v_m + 0.5 \right)
    \end{equation}

    where $O(m,n)$ indicates the $mn$-th pixel ($m$-th row, $n$-th column element) in the image and $[u_n, v_m]^\top$ are the image coordinates of the $mn$-th pixel (see Figure \ref{fig:silh_scheme}).
    
\end{definition}

From Definition \ref{def:O_img}, the digital image of the observed silhouette can be interpreted as a Boolean occupancy grid where each pixel containing a portion of the observed surface is equal to 1 and all other pixels are equal to 0. Effectively, $O(m,n;\phi)$ is a discrete representation of the continuous region $\mathcal{O}(u,v;\phi)$.

\begin{definition}
\label{def:O_bar}
    Given a set of camera longitudes $\Phi = \{ \phi_1, \dots, \phi_M \}$ and the corresponding observed-silhouette images $\{ O(\phi_1), \dots, O(\phi_M) \}$, the silhouette-stack image $\bar{O}(m,n;\Phi)$ is defined as:

    \begin{equation}
         \bar{O}(m,n;\Phi) = \sum_{k=1}^M O(m,n; \phi_k)
     \end{equation}
     
\end{definition}

The following key assumption must be made in order to apply the continuous-silhouette model to discrete silhouette imagery.

\begin{assumption}
    \label{assumpt:O=O_mathcal}
    We assume that the silhouette-stack image $\bar{O}(m,n;\Phi)$ can be approximated as

    \begin{equation}
        \bar{O}(m,n;\Phi) \approx \tilde{\bar{\mathbf{1}}}_\mathcal{O}(u_n,v_m; \phi_1, \phi_M)
    \end{equation}
\end{assumption}

Assumption \ref{assumpt:O=O_mathcal} posits that the silhouette stack derived from a finite set of digital images closely matches its continuous-function counterpart. For this approximation to remain accurate, sufficiently high image resolution and adequate coverage in camera longitude are necessary, though the exact requirements depend on the specific scenario. This assumption is subsequently validated through experiments that employ realistic resolution and imaging-cadence values (Section \ref{sec:results}).

\begin{lemma}
    \label{lemma:O=X+E}
    From Theorem \ref{th:1_O=1_O_symm+e}, Definition \ref{def:O_bar}, and Assumption \ref{assumpt:O=O_mathcal}, the silhouette-stack image $\bar{O}(m,n;\Phi)$ can be written as:

    \begin{equation}
        \bar{O}(m,n;\Phi) = \bar{O}_\mathrm{symm}(m,n;\Phi) + Z(m,n;\Phi),\, m,n=1,\dots,N
    \end{equation}

    where $\bar{O}_\mathrm{symm}(m,n;\Phi)$ is a symmetric image with respect to $\boldsymbol{\omega}_\mathrm{proj}$ and $Z(m,n;\Phi)$ is the image due to error sources.
\end{lemma}

\begin{proof}
    The proof is straightforward.
\end{proof}

\subsection{Image Symmetry Metrics}
\label{sec:image_symm_metrics}

In practice, silhouette-stack images rarely exhibit perfect symmetry, due to error sources such as surface shadowing and silhouette registration, as previously discussed. In this section, we introduce the image-symmetry metric used in the proposed technique to estimate the pole projection onto the camera plane. More details on the specific operations involved are reported in the appendix.

\begin{definition}
    \label{def:chi}
    Let $I_a \in \mathbb{R}^{N\times N}$ and $I_b \in \mathbb{R}^{N\times N}$ be two images. The normalized correlation coefficient $\chi$ between $I_a$ and $I_b$ is defined as:

\begin{equation}
\label{eq:chi}
    \chi(I_a,I_b) = \dfrac{\sum_{m=1}^{N}\sum_{n=1}^{N}(I_{a,mn}-\bar{I}_a)(I_{b,mn}-\bar{I}_b) }{\sqrt{\left(\sum_{m=1}^{N}\sum_{n=1}^{N} (I_{a,mn}-\bar{I}_a)^2 \right)\left(\sum_{m=1}^{N}\sum_{n=1}^{N} (I_{b,mn}-\bar{I}_b)^2 \right)}}
\end{equation}

where $\bar{I}_a$ and $\bar{I}_b$ are the mean pixel intensities for $I_a$ and $I_b$, respectively, i.e.:

\begin{equation}
    \bar{I}_i = \dfrac{1}{N^2} \sum_{m=1}^{N}\sum_{n=1}^{N} I_{i,mn},\,i = a,b.
\end{equation}
\end{definition}

\begin{definition}
    \label{def:psi_I}
    Given an image $I\in \mathbb{R}^{N\times N}$, its vertical-symmetry score $\psi(I)$ is given by:

    \begin{equation}
        \psi(I) = \chi(I,\underline{I})
    \end{equation}

    where $\underline{I}$ is the image obtained by reflecting $I$ about its vertical axis (Definition \ref{def:I_ref}).
\end{definition}

\begin{lemma}
        \label{th:argmax_alpha}
        Let $I\in\mathbb{R}^{N\times N}$ be an image which exhibits reflective symmetry with respect to the axis-of-symmetry angle $\alpha$ (Definition \ref{def:image_refl_symm_theta}). Also let $I_\theta$ be the image obtained by applying the rotation $R(\theta)$ to $I$ (see Definition \ref{def:I_theta}). Then,

            \begin{equation}
            \label{eq:pi/2-alpha+kpi}
        \{(\pi/2-\alpha) + k\pi,\, k\in\mathbb{Z} \} = \argmax_{\theta} \psi(I_\theta)
    \end{equation}
    
\end{lemma}

\begin{proof}

    Since $I$ exhibits reflective symmetry with respect to $\alpha$, it is easy to show that applying a rotation $R(\pi/2-\alpha)$ to $I$ produces another image $I_{\pi/2-\alpha}$ such that $\psi(I_{\pi/2-\alpha}) = \mathrm{max}(\psi) = 1$, i.e., $I_{\pi/2-\alpha}$ exhibits reflective symmetry with respect to the vertical axis. Since reflective symmetry about $\alpha$ also implies reflective symmetry about $\alpha+k\pi,\,k\in\mathbb{Z}$, Equation \ref{eq:pi/2-alpha+kpi} is obtained. 
\end{proof}

Note that if $I$ exhibits multiple axes of symmetry, $Equation \ref{eq:pi/2-alpha+kpi}$ holds true for angles other than $\alpha$ as well.

\subsection{Symmetry Detection in the Frequency Domain}
\label{sec:fourier}

Our approach to image-symmetry detection leverages frequency domain analysis of silhouette-stack images. This section presents the theoretical foundations and rationale for this methodology.

\subsubsection{Motivation}

The symmetry of silhouette-stack images can be compromised by several factors, including partial surface coverage, shadowing effects, and image registration errors (Sections \ref{sec:perfect_silh} and \ref{sec:imaging_error_sources}). Additionally, uncertainty in the object's center of mass position presents a fundamental challenge for symmetry detection. To address these limitations, we propose a frequency-domain approach based on the Discrete Fourier Transform (DFT)\cite{brigham1988fast}. Our method operates on the amplitude spectrum of the DFT (Definition \ref{def:dft}), which maintains the reflective symmetry properties of the original image while achieving translation invariance. Previous research has demonstrated the effectiveness of DFT-based approaches for symmetry detection in real images containing complex structures and imperfect symmetries\cite{keller2006signal}. The mathematical foundations and key properties of the DFT relevant to our analysis are provided in the appendix. The DFT can be computed using efficient algorithms, most notably the Fast Fourier Transform (FFT)\cite{brigham1988fast}. Further, previous work has presented efficient image representations for DFT-based symmetry detection, such as the pseudopolar Fourier Transform\cite{keller2006signal,bermanis20103}. In this work, we use a standard image representation and FFT to compute the DFT.

\subsubsection{Properties of Fourier Amplitude Spectra}

In this work, we leverage three properties of DFT amplitude spectra: symmetry, translation invariance, and robustness to noise. We start by introducing known symmetry properties without proof.

\begin{lemma}
    \label{lemma:conj_symm}
    
        Let $I\in \mathbb{R}^{N\times N}$ be a real-valued image and let $F=\mathcal{F}(I)\in \mathbb{C}^{N\times N}$ be its DFT (Definition \ref{def:dft}). Then, $F$ exhibits conjugate symmetry, i.e.:

    \begin{equation}
        F(x,y) = F^*(-x,-y),\forall \; x,y
    \end{equation}

    where $F^*$ denotes the complex conjugate of $F$ and $(x,y)$ are the pixel coordinates with respect to the image center.
    
\end{lemma}

\begin{corollary}
     \label{cor:A_symm}
     Let $A = |\mathcal{F}(I)|\in\mathbb{R}^{N\times N}$ be the DFT amplitude spectrum of $I$ (\ref{def:dft}). Then, $A$ exhibits central symmetry, i.e.:

     \begin{equation}
        A(x,y) = A(-x,-y),\, x,y = 1,\dots,N
    \end{equation}
    
\end{corollary}

\begin{lemma}
    \label{lemma:A_symm}
    If $I\in\mathbb{R}^{N\times N}$ exhibits reflective symmetry with respect to the axis-of-symmetry angle $\alpha$, then its DFT amplitude spectrum $A$ also exhibits reflective symmetry with respect to $\alpha$.
\end{lemma}

\begin{corollary}
\label{cor:A_symm_pi/2}
    If $I\in\mathbb{R}^{N\times N}$ exhibits reflective symmetry with respect to the axis-of-symmetry angle $\alpha$, then $A$ exhibits reflective symmetry with respect to $\alpha + k\dfrac{\pi}{2},\, k\in \mathbb{Z}$.
\end{corollary}

\begin{proof}
    Corollary \ref{cor:A_symm_pi/2} can be shown by construction, by combining central symmetry (Corollary \ref{cor:A_symm}) with reflective symmetry (Lemma \ref{lemma:A_symm}).
\end{proof}

As shown in Corollary \ref{cor:A_symm_pi/2}, when an image $I$ exhibits reflective symmetry, two equivalent axis-of-symmetry solutions emerge in the amplitude spectrum $A$. This property has important implications for both symmetry detection and pole estimation. We provide a more detailed discussion of this phenomenon in Section \ref{sec:pole_ambig}, where we also present practical strategies for resolving the resulting ambiguity.

We now examine the translation-invariant property of the DFT amplitude spectrum, beginning with related definitions.

\begin{definition}
\label{def:circular_shift}
    Given an image $I\in\mathbb{R}^{N\times N}$, the image obtained by applying a circular-shift translation $\mathbf{t}_\circ = [t_{\circ,u}, t_{\circ,v}]^\top$ is an image $I_{\mathbf{t}_\circ}\in\mathbb{R}^{N\times N}$ such that:

    \begin{equation}
    \label{eq:I_t}
        I_{\mathbf{t}_\circ}(m, n) = I\left((m - t_{\circ,v}) \, \bmod \, N, (n - t_{\circ,u}) \, \bmod \, N \right)
    \end{equation}
    
\end{definition}

A circular shift (Definition \ref{def:circular_shift}) wraps pixels that extend beyond image boundaries to the opposite side, maintaining the signal's cyclical properties across both horizontal and vertical dimensions. The modulo operator ($\bmod$) constrains pixel indices to the valid range of $1,\dots,N$. This circular shift property has direct relevance to our imaging scenario: when an image has uniform (or removed) background, object translations within the camera's field of view can be mathematically represented as circular shifts, provided the object remains fully contained within the field of view as specified in Section \ref{sec:assumptions}.

\begin{theorem}
    \label{th:shift}
    Let $I_{\mathbf{t}_\circ}\in \mathbb{R}^{N\times N}$ be the image obtained by applying a circular-shift translation $\mathbf{t}_\circ$ to the image $I\in \mathbb{R}^{N\times N}$. Let $A$ and $A_{\mathbf{t}_\circ}$ be the DFT amplitude spectra of $I$ and $I_{\mathbf{t}_\circ}$, respectively. Then:

    \begin{equation}
        A = A_{\mathbf{t}_\circ}
    \end{equation}
\end{theorem}

Theorem \ref{th:shift}, commonly known as the Fourier Transform Shift Theorem, establishes a fundamental property of DFT amplitude spectra\cite{roberts1987digital}. While image translation affects the phase spectrum (see Definition \ref{def:dft}), the amplitude spectrum remains invariant under translation. This invariance property is particularly valuable for space-based pole estimation, as it eliminates the need to precisely determine the object's center-of-mass location in the image. When analyzing silhouette-stack images, the DFT amplitude spectrum---and by extension, any symmetry properties---remains unchanged regardless of translation. It is important to note that this translation invariance property holds only when the stacked silhouette is fully contained within the camera's field of view, consistent with the assumptions outlined in Section \ref{sec:assumptions}.

The DFT amplitude spectrum exhibits inherent robustness to image noise due to its global representational properties. In the frequency domain, global spatial patterns such as symmetries manifest primarily in low-frequency components, while noise typically corresponds to high-frequency components. This natural frequency-domain separation enhances our ability to distinguish symmetry patterns from asymmetric image perturbations. The global nature of the Fourier Transform thus provides an effective framework for isolating and analyzing symmetric structures within noisy images\cite{keller2006signal}.

\subsubsection{Symmetry Detection with Noisy Images}

We can now derive results for DFT-based image symmetry detection in the presence of error sources. We will rely on the symmetry score in Definition \ref{def:psi_I} as a symmetry metric and the image error model obtained in Theorem \ref{th:1_O=1_O_symm+e}.

\begin{lemma}
\label{lemma:A_symm_plus_blob}
    Let $I\in\mathbb{R}^{N\times N}$ be an image such that

    \begin{equation}
        I = X + E
    \end{equation}
    
    where the image $X \in \mathbb{R}^{N\times N}$ is symmetric with respect to the axis-of-symmetry angle $\alpha$ and $E\in \mathbb{R}^{N\times N}$ is an asymmetric image (in the same form as shown in Lemma \ref{lemma:O=X+E}). Also let $F = \mathcal{F}(I)\in\mathbb{C}^{N\times N}$ be the DFT of $I$ and $A=|F|$ be the corresponding DFT amplitude spectrum (Definition \ref{def:dft}). Then, the squared amplitude spectrum $A^2$ can be written as

    \begin{equation}
        A^2 = A^2_X + A^2_N
    \end{equation}

    where $A^2_{X}$ is a symmetric image with respect to $\alpha$ and $A^2_N$ is an asymmetric image.
    
\end{lemma}

\begin{proof}

    From linearity of the DFT (Lemma \ref{lemma:dft_linearity}), we have:

    \begin{align}
        F &= \mathcal{F}(X + E)\\
        &= \mathcal{F}(X) + \mathcal{F}(E)\\
        &= F_X + F_E
    \end{align}

    where $F_X$ and $F_E$ are used to denote $\mathcal{F}(X)$ and $\mathcal{F}(E)$, for brevity. Using the properties of complex numbers, the squared amplitude spectrum $A^2$ can be rewritten as:

    \begin{align}
        A^2 &= |F_{X} + F_E|^2\\
        &= \left[\mathrm{Re}(F_{X} + F_E)\right]^2 + \left[\mathrm{Im}(F_{X} + F_E)\right]^2\\
        &= \left[\mathrm{Re}(F_{X})\right]^2 + \left[\mathrm{Re}(F_E)\right]^2 + 2\mathrm{Re}(F_{X})\mathrm{Re}(F_E) \nonumber \\
        & \quad + \left[\mathrm{Im}(F_{X})\right]^2 + \left[\mathrm{Im}(F_E)\right]^2 + 2\mathrm{Im}(F_{X})\mathrm{Im}(F_E)\\
        &= |F_{X}|^2 + |F_E|^2 + 2\left( \mathrm{Re}(F_{X})\mathrm{Re}(F_E) + \mathrm{Im}(F_{X})\mathrm{Im}(F_E) \right)\\
        &= A^2_{X} + A^2_N
    \end{align}

    where

    \begin{equation}
        A^2_{X} = |F_X|^2
    \end{equation}

    and

    \begin{equation}
         A^2_N = |F_E|^2 + 2\left( \mathrm{Re}(F_{X})\mathrm{Re}(F_E) + \mathrm{Im}(F_{X})\mathrm{Im}(F_E) \right)
    \end{equation}

    Since $X$ is symmetric, its squared amplitude spectrum $A^2_{X}$ is also symmetric (Lemma \ref{lemma:A_symm}).
\end{proof}

Lemma \ref{lemma:A_symm_plus_blob} characterizes how asymmetric components affect the DFT amplitude spectrum. While the transformation to the frequency domain introduces coupling between symmetric and asymmetric elements via the term $A^2_N$, a distinct symmetric component is preserved through the term $A^2_X$. This preservation of symmetry in the frequency domain is fundamental to our DFT-based symmetry-detection approach.

To further analyze the effect of asymmetric error sources on symmetry detection, we make the following assumption.

\begin{assumption}
\label{assumption:const_xcorr}
    Given the squared amplitude spectrum model $A^2 = A^2_{X} + A^2_N$ presented in Lemma \ref{lemma:A_symm_plus_blob}, the noise pixels $A^2_N(m,n)$ are random variables not exhibiting reflective symmetry with respect to any axis-of-symmetry angle $\theta$, in the sense that

    \begin{equation}
        \mathbb{E}\left[ \left(A^2_{N,\theta}(m,n) - \widebar{A^2_N} \right) \left(\underline{A^2_{N,\theta}}(m,n) - \widebar{A^2_N} \right) \right] = 0,\; m,n=1\dots,N,\; \theta=[0,2\pi)
    \end{equation}
    
where $A^2_{N,\theta}$ is the image obtained by rotating $A^2_N$ by an angle $\theta$  (Definition \ref{def:I_theta}), $\underline{A^2_{\theta,N}}$ is the image obtained by reflecting $A^2_{\theta,N}$ about the vertical axis (Definition \ref{def:I_ref}), and $\widebar{A^2_N}=\frac{1}{N}\sum_{m=1}^N\sum_{n=1}^N A^2_N(m,n)$ is the mean pixel intensity in $A^2_N$.
\end{assumption}

Assumption \ref{assumption:const_xcorr} is supported by the nature of the noise term $A^2_N$, which encompasses the cumulative asymmetric effects in the silhouette-stack image---including partial surface coverage, surface shadowing, and image-registration errors described in Sections \ref{sec:perfect_silh} and \ref{sec:imaging_error_sources}. When accumulated over a sufficient number of stacked silhouettes, these combined effects can be modeled as random frequency components distributed across the frequency domain without dominant modes. This assumption may not hold when the stacked silhouette exhibits symmetric patterns that are not aligned with the pole direction. Two notable cases illustrate potential violations: first, high-contrast lines resulting from silhouette stacking may appear asymmetrically if shadowing obscures them on one side of the pole. Second, a terminator line (separating lit and shadowed pixels) may manifest as a continuous linear feature rather than a fragmented pattern. However, these effects typically remain minor for objects with irregular shapes. Furthermore, non-polar symmetries tend to become dispersed during the stacking of multiple irregular silhouettes.

Using the above mathematical setup, we can state the main result for symmetry detection using DFT amplitude spectra in the presence of noise.

\begin{theorem}
    \label{th:alpha_argmax_noise}
    Let $I = X + E,\,I\in\mathbb{R}^{N\times N}$ be the image with symmetric component $X$ and asymmetric component $E$, according to Lemma \ref{lemma:A_symm_plus_blob}. Also let $A=|\mathcal{F}(I)|$ be the DFT amplitude spectrum of $I$. Then,

    \begin{equation}
    \label{eq:alpha_argmax}
        \left\{\left(\dfrac{\pi}{2}-\alpha\right) + k\dfrac{\pi}{2},\, k\in\mathbb{Z} \right\} = \argmax_\theta \mathbb{E}\left[\psi(A^2_\theta)\right]
    \end{equation}

    where $A^2_\theta$ is the image obtained by rotating $A^2$ by an angle $\theta$ (Definition \ref{def:I_theta}).
\end{theorem}

\begin{proof}
    From Lemma \ref{lemma:A_symm_plus_blob}, the symmetry score $\psi$ in Equation \ref{eq:alpha_argmax} can be written as:

    \begin{align}
        \psi(A^2_\theta) &= \psi(A^2_{\theta,X} + A^2_{\theta,N})\\
        &= \chi(A^2_{\theta,X} + A^2_{\theta,N},\underline{A^2_{\theta,X}} + \underline{A^2_{\theta,N}}) \label{eq:psi_2}
    \end{align}
    
    We are interested in studying $\mathbb{E}\left[\psi(A^2_\theta)\right]$ as a function of $\theta$. Applying Definition \ref{def:psi_I}, we have:

    \begin{align}
    \label{eq:psi_A^2_theta}
        \psi(A^2_\theta) &= \dfrac{1}{\xi_{X,N}} \sum_{m=1}^N \sum_{n=1}^N \left(A^2_{\theta,X}(m,n) + A^2_{\theta,N}(m,n) - \widebar{A_\theta^2} \right) \nonumber \\ 
        & \quad\quad\quad\quad\quad\quad\quad \left( \underline{A^2_{\theta,X}}(m,n) + \underline{A^2_{\theta,N}}(m,n) - \widebar{A_\theta^2} \right)
    \end{align}

    where $\xi_{X,N} \in \mathbb{R}$ is the cross-correlation denominator in Equation \ref{eq:chi}, not varying with $\theta$ and hence considered constant here, and $\widebar{A_\theta^2}$ is the mean pixel intensity of $A^2_\theta$ (as well as $\underline{A^2_\theta}$). From linearity, $\widebar{A_\theta^2}$ can be written as

    \begin{equation}
        \widebar{A_\theta^2} = \widebar{A_{\theta,X}^2} + \widebar{A_{\theta,N}^2}
    \end{equation}
    
    where $\widebar{A_{\theta,X}^2}$ and $\widebar{A_{\theta,N}^2}$ are the mean pixel intensities of $A_{\theta,X}^2$ and $A_{\theta,N}^2$, respectively. If we define the mean-subtracted images $\grave{A}^2_{\theta,X}$ and $\grave{A}^2_{\theta,N}$ such that
    
        \begin{equation}
        \grave{A}^2_{\theta,X}(m,n) =  A^2_{\theta,X}(m,n) - \widebar{A_{\theta,X}^2}
    \end{equation}

    \begin{equation}
        \grave{A}^2_{\theta,N}(m,n) =  A^2_{\theta,N}(m,n) - \widebar{A_{\theta,N}^2}
    \end{equation}
    
then, Equation \ref{eq:psi_A^2_theta} can be rewritten as

    \begin{align}
        \psi(A^2_\theta) &= \dfrac{1}{\xi_{X,N}} \sum_{m=1}^N \sum_{n=1}^N \left[\grave{A}^2_{\theta,X}(m,n) + \grave{A}^2_{\theta,N}(m,n) \right] \left[ \underline{ \grave{A}^2_{\theta,X}}(m,n) + \underline{\grave{A}^2_{\theta,N}}(m,n) \right]\\
        &= \dfrac{1}{\xi_{X,N}} \sum_{m=1}^N \sum_{n=1}^N \left[ \grave{A}^2_{\theta,X}(m,n) \underline{\grave{A}^2_{\theta,X}}(m,n) + \grave{A}^2_{\theta,X}(m,n) \underline{\grave{A}^2_{\theta,N}}(m,n) + \nonumber \right. \\
        & \quad\quad\quad\quad\quad\quad\quad \left. \grave{A}^2_{\theta,N}(m,n) \underline{\grave{A}^2_{\theta,X}}(m,n) + \grave{A}^2_{\theta,N}(m,n) \underline{\grave{A}^2_{\theta,N}}(m,n) \right]
    \end{align}

Using the linearity of expectation, we can write:

\begin{align}
\label{eq:E[psi]}
    \mathbb{E}\left[ \psi(A^2_\theta) \right] &= \dfrac{1}{\xi_{X,N}} \sum_{m=1}^N \sum_{n=1}^N \mathbb{E}\left[ \grave{A}^2_{\theta,X}(m,n) \underline{\grave{A}^2_{\theta,X}}(m,n) \right] + \mathbb{E}\left[ \grave{A}^2_{\theta,X}(m,n) \underline{\grave{A}^2_{\theta,N}}(m,n) \right] + \nonumber \\
        & \quad\quad\quad\quad\quad\quad\quad \mathbb{E}\left[ \grave{A}^2_{\theta,N}(m,n) \underline{\grave{A}^2_{\theta,X}}(m,n) \right] + \mathbb{E}\left[ \grave{A}^2_{\theta,N}(m,n) \underline{\grave{A}^2_{\theta,N}}(m,n) \right]
\end{align}

The first term is the product between the deterministic, symmetric image and its reflected counterpart; hence:

\begin{equation}
    \mathbb{E}\left[ \grave{A}^2_{\theta,X}(m,n) \underline{\grave{A}^2_{\theta,X}}(m,n) \right] = \grave{A}^2_{\theta,X}(m,n) \underline{\grave{A}^2_{\theta,X}}(m,n)
\end{equation}

The second and third terms are the product between the signal and the random-noise component. Since $\underline{\grave{A}^2_{\theta,N}}$ is the mean-subtracted random variable, by definition of expectation we have:

\begin{align}
    \mathbb{E}\left[ \grave{A}^2_{\theta,X}(m,n) \underline{\grave{A}^2_{\theta,N}}(m,n) \right] &= \grave{A}^2_{\theta,X}(m,n)\, \cdot \mathbb{E}\left[ 
 \underline{\grave{A}^2_{\theta,N}}(m,n) \right] \\
 &= \grave{A}^2_{\theta,X}(m,n) \cdot 0 = 0
\end{align}

\begin{align}
    \mathbb{E}\left[ \grave{A}^2_{\theta,N}(m,n) \underline{\grave{A}^2_{\theta,X}}(m,n) \right] &= \mathbb{E}\left[ \grave{A}^2_{\theta,N}(m,n) \right] \cdot \underline{\grave{A}^2_{\theta,X}}(m,n) \\
 &= 0 \cdot \underline{\grave{A}^2_{\theta,X}}(m,n) = 0
\end{align}

Lastly, the fourth term is the product between the random-noise component and its reflected counterpart. Using Assumption \ref{assumption:const_xcorr}, we can write:

\begin{equation}
    \mathbb{E}\left[ \grave{A}^2_{\theta,N}(m,n) \underline{\grave{A}^2_{\theta,N}}(m,n) \right] = 0
\end{equation}

Combining the expectation results from each term, Equation \ref{eq:E[psi]} becomes:

\begin{align}
\label{eq:exp_psi}
    \mathbb{E}\left[ \psi(A^2_\theta) \right] &= \dfrac{1}{\xi_{X,N}} \sum_{m=1}^N \sum_{n=1}^N \grave{A}^2_{\theta,X}(m,n) \underline{\grave{A}^2_{\theta,X}}(m,n)\\
    &= \label{eq:exp_psi_2} \dfrac{\xi_X}{\xi_{X,N}} \psi(A_{\theta,X}^2)
\end{align}

where $\xi_X \in \mathbb{R}$ is the cross-correlation denominator in the expression of $\psi(A_{\theta,X}^2)$. The term $\xi_X/\xi_{X,N}$ is constant and hence does not alter the maxima in Equation \ref{eq:exp_psi_2}. $\psi(A_{\theta,X}^2)$ is the symmetry score of the deterministic component $A_{\theta,X}^2$, which is symmetric with respect to $\alpha$. Then, by applying Lemma \ref{th:argmax_alpha} and Corollary \ref{cor:A_symm_pi/2} to $\psi(A_{\theta,X}^2)$, we observe that Theorem \ref{th:alpha_argmax_noise} holds true.
\end{proof}

Theorem \ref{th:alpha_argmax_noise} establishes the theoretical foundation for our proposed pole estimation algorithm. Specifically, it demonstrates that the DFT amplitude spectrum enables extraction of the symmetric component from noise-corrupted silhouette stacks, in principle providing robust estimates of the axis-of-symmetry angle $\alpha$ and, consequently, the pole projection $\boldsymbol{\omega}_\mathrm{proj}$.

\subsubsection{Pole Ambiguities}
\label{sec:pole_ambig}

A silhouette-stack image with reflective symmetry about the projected-pole direction exhibits symmetry with respect to angles $\alpha + k\pi\,,k\in\mathbb{Z}$, creating a two-fold ambiguity in pole direction determination. This ambiguity is further compounded in the DFT amplitude spectrum, which displays reflective symmetry about both the original direction and its orthogonal direction, as demonstrated by Corollary \ref{cor:A_symm_pi/2}. The resulting symmetry with respect to angles $\alpha + k\pi/2\,,k\in\mathbb{Z}$ introduces a four-fold ambiguity in pole direction estimation.
Resolving these ambiguities is essential for accurate pole estimation. Several practical approaches can address this challenge. Visual inspection of imagery provides one direct solution. Alternatively, image-processing techniques can serve as initialization steps. Feature-based approaches can estimate surface point motion relative to the camera, with dense optical flow \cite{kroeger2016fast} being suitable for low-resolution images and feature tracking \cite{lucas1981iterative} for high-resolution cases. Under a hovering-camera model with known inertial attitude (Section \ref{sec:assumptions}), these feature-based algorithms can provide direction-of-motion measurements \cite{christian2021image} to identify the most probable pole hypothesis based on surface-relative camera motion.
In many practical scenarios, a-priori pole direction information is available from ground-based lightcurve observations \cite{kaasalainen2001optimization}, potentially eliminating the need for initialization procedures. An alternative approach involves maintaining multiple pole hypotheses throughout the estimation process, though this increases computational complexity. For the subsequent analysis, we assume the actual pole direction has been successfully identified from among the four possibilities.

\subsection{Pole Triangulation}
\label{sec:pole_triangulation}

Prior sections established the theoretical framework for estimating pole projections onto individual camera planes, where each projection $\boldsymbol{\omega}_\mathrm{proj}$ and its associated axis-of-symmetry angle $\alpha$ correspond to a specific hovering-camera view. However, determining the complete three-dimensional pole orientation $\boldsymbol{\omega}$, which has two degrees of freedom, requires multiple projections. We propose a method called \textit{pole triangulation} that resolves the pole direction by utilizing changes in camera-boresight orientation, or equivalently, changes in the camera plane. This approach synthesizes multiple in-plane angle estimates $\alpha$ obtained from different camera-boresight orientations to serve as indirect measurements of $\boldsymbol{\omega}$. While approaching a space object, natural variations in camera latitude typically occur, providing opportunities to capture silhouette image sets from multiple viewing orientations and facilitating effective pole triangulation through observer motion. We formulate pole triangulation as a least-squares problem, as described below.

\begin{theorem}
    \label{th:pole_triang}
    Let $\alpha_1, \dots, \alpha_\aleph$ be a set of angles describing the projected-pole directions (see Definition \ref{def:alpha=atan2}) observed from the corresponding hovering camera-view set $ \mathcal{V}_1, \dots, \mathcal{V}_\aleph$. Also let $(\mathbf{i}_{\mathcal{C}_j},\mathbf{j}_{\mathcal{C}_j},\mathbf{k}_{\mathcal{C}_j})$ be the unit vectors representing the x, y, and z axes of the $j$-th camera reference frame $\mathcal{C}_j$. Then, the 3D pole direction $\boldsymbol{\omega}$ can be estimated by solving the least-squares problem:

    \begin{equation}
    \label{eq:A omega = 0}
        M \,\boldsymbol{\omega} = 0
    \end{equation}

    where $M\in \mathbb{R}^{2\aleph \times 3}$ is given by

    \begin{equation}
    \label{eq: A = sin_alpha}
        M = \left[ \mathrm{sin}(\alpha_1)\mathbf{i}^\top_{\mathcal{C}_1},\, - \mathrm{cos}(\alpha_1)\mathbf{j}^\top_{\mathcal{C}_1},\, \dots,\, \mathrm{sin}(\alpha_\aleph)\mathbf{i}^\top_{\mathcal{C}_\aleph},\, - \mathrm{cos}(\alpha_\aleph)\mathbf{j}^\top_{\mathcal{C}_\aleph} \right]^\top.
    \end{equation}
    
\end{theorem}

\begin{proof}
    Observe that the $j$-th pole projection along the camera plane, $\boldsymbol{\omega}'_{\mathrm{proj},j}$ (Equation \ref{eq:omega_proj=omega'/norm_omega'}), can be rewritten as:

    \begin{equation}
        \boldsymbol{\omega}'_{\mathrm{proj},j} = \left( \boldsymbol{\omega}^\top \mathbf{i}_{\mathcal{C}_j} \right)\mathbf{i}_{\mathcal{C}_j} + \left( \boldsymbol{\omega}^\top \mathbf{j}_{\mathcal{C}_j} \right)\mathbf{j}_{\mathcal{C}_j}
    \end{equation}

    Then, the in-plane pole-projection angle $\alpha_j$ can be related to the projected-pole components as:

    \begin{equation}
        \mathrm{tan}(\alpha_j) = \dfrac{\mathrm{sin}(\alpha_j)}{\mathrm{cos}(\alpha_j)} = \dfrac{\left( \boldsymbol{\omega}^\top \mathbf{j}_{\mathcal{C}_j} \right)\mathbf{j}_{\mathcal{C}_j}}{\left( \boldsymbol{\omega}^\top \mathbf{i}_{\mathcal{C}_j} \right)\mathbf{i}_{\mathcal{C}_j}}
    \end{equation}

    which can be rearranged as

    \begin{equation}
    \label{eq:sin-alpha_i_omega_i}
        \left( \boldsymbol{\omega}^\top \mathbf{i}_{\mathcal{C}_k} \right)\mathrm{sin}(\alpha_j)\mathbf{i}_{\mathcal{C}_j} - \left( \boldsymbol{\omega}^\top \mathbf{j}_{\mathcal{C}_j} \right)\mathrm{cos}(\alpha_j)\mathbf{j}_{\mathcal{C}_j} = 0.
    \end{equation}

    Rewriting Equation \ref{eq:sin-alpha_i_omega_i} in matrix form, we obtain:

    \begin{equation}
        \begin{bmatrix}
            \mathrm{sin}(\alpha_j)\mathbf{i}^\top_{\mathcal{C}_j} \\
            - \mathrm{cos}(\alpha_j)\mathbf{j}^\top_{\mathcal{C}_j}
        \end{bmatrix} \boldsymbol{\omega} = 0.
    \end{equation}

    Let $M_j=\left[ \mathrm{sin}(\alpha_j)\mathbf{i}^\top_{\mathcal{C}_j}, - \mathrm{cos}(\alpha_j)\mathbf{j}^\top_{\mathcal{C}_j} \right]^\top$. Then, the overdetermined linear system of equations in Equations \ref{eq:A omega = 0} and \ref{eq: A = sin_alpha} is obtained by vertically stacking matrices $M_1, \dots, M_\aleph$ corresponding to the angle observations $\alpha_1, \dots, \alpha_\aleph$. Applying Singular Value Decomposition (SVD)\cite{golub2013matrix}, one can write

    \begin{equation}
        M = U \, \Sigma \, V^\top
    \end{equation}

    where $U \in \mathbb{R}^{\aleph \times \aleph}$ and $V \in \mathbb{R}^{3 \times 3}$ are orthonormal matrices and $\Sigma \in \mathbb{R}^{\aleph \times 3}$ is a rectangular diagonal matrix of singular values. The estimate of $\boldsymbol{\omega}$, up to scale, is given by the last column of $V$.
    
\end{proof}

The least-squares solution from Theorem \ref{th:pole_triang} does not constrain the pole direction $\boldsymbol{\omega}$ to unit norm. Therefore, after estimation, $\boldsymbol{\omega}$ must be normalized to obtain a valid direction vector. Also observe that the hovering views $\mathcal{V}_1, \dots, \mathcal{V}_\aleph$ need not be located at the same radial distance from the target's center in order for pole triangulation to be effective.


\section{Algorithm Overview}
\label{sec:algorithm_overview}

The PoleStack algorithm builds upon the principles reported in Section \ref{sec:theor_dev} and is described below. A set of image batches $\{ \mathcal{I}_1, \dots, \mathcal{I}_\aleph \}$ is given, where the batch $\mathcal{I}_j=\{ I_{j,1}, \dots, I_{j,M_j} \}$ is collected from the hovering-camera viewset $\mathcal{V}_j$ (see Section \ref{sec:problem_formulation}). Suppose that images within a batch have been co-registered---e.g., by aligning their center of brightness---and that a silhouette image $O_{j,k}\in\mathbb{B}^{N\times N}$ (Definition \ref{def:O_img}) has been previously extracted from the corresponding image $I_{j,k}\in\mathbb{R}^{N\times N}$. The PoleStack algorithm is divided into two high-level steps:

\begin{enumerate}
    \item For each image batch $\mathcal{I}_j$, estimate the pole-projection angle $\alpha_j$ by finding the direction of maximum symmetry in the silhouette-stack image $\bar{O}_j$, using Algorithm \ref{alg:in-plane}. The angle estimate is denoted as $\hat{\alpha}_j$.
    \item Given a set of pole-projection angle estimates $\{ \hat{\alpha}_1, \dots, \hat{\alpha}_\aleph \}$, estimate the 3D pole direction $\boldsymbol{\omega}$ by solving Equations \ref{eq:A omega = 0} and \ref{eq: A = sin_alpha} (pole triangulation). The pole-direction estimate is denoted as $\hat{\boldsymbol{\omega}}$.
\end{enumerate}

\begin{algorithm}
\caption{Pole-projection angle estimation}
\label{alg:in-plane}
\begin{algorithmic}[1]
\Require $\{O_{j,1}, \dots, O_{j,M}\},\, O_{j,k}\in\mathbb{B}^{N\times N}$ \Comment{Observed silhouette images from camera viewset $\mathcal{V}_j$ (Definition \ref{def:O_img})}
\Require $\tau\in\mathbb{R}$ \Comment{Cutoff frequency for the DFT amplitude spectrum (Definition \ref{def:circ_cropping})}
\Require $\Theta=\{ \theta_1, \dots, \theta_{N_\theta} \},\, \theta_i\in\mathbb{R}$ \Comment{Query angles for symmetry evaluation}
\Ensure $\hat{\alpha}$ \Comment{Pole-projection angle estimate}
\State $\bar{O}_j(m,n) \gets \sum_{k=1}^M O_{j,k}(m,n),\, m,n = 1,\dots,N$ \Comment{Compute silhouette-stack image}
\State $A \gets |\mathcal{F}(\bar{O}_j)|$ \Comment{Compute DFT amplitude spectrum (Definition \ref{def:dft})}
\State $A_\mathrm{filt} \gets \mathrm{Crop}(A; \tau)$ \Comment{Apply low-pass filter using circular crop (Definition \ref{def:circ_cropping})}
\State $E_\mathrm{filt} \gets \mathrm{log}(1 + A_\mathrm{filt}^2)$ \Comment{Compress image by computing the log-power spectrum} \label{state:log}
\For{$i = 1,\dots,N_{\theta}$}
    \State $E_{\mathrm{filt},\theta_i} \gets \mathrm{Rot}(E_\mathrm{filt}; \theta_i)$ \Comment{Rotate image $E_\mathrm{filt}$ by an angle $\theta_i$ (Definition \ref{def:I_theta})}
    \State $\beta_i \gets \psi(E_{\mathrm{filt},\theta_i})$ \Comment{Evaluate and log $i$-th symmetry score (Definition \ref{def:psi_I})}
\EndFor
\State $\hat{\alpha} \gets \argmax\limits_{\Theta} \{ \beta_1, \dots, \beta_{N_\theta} \}$ \Comment The estimate $\hat{\alpha}$ is the angle $\theta_i$ which maximizes reflective symmetry

\end{algorithmic}
\end{algorithm}


Algorithm \ref{alg:in-plane} returns the direction of maximum symmetry within the silhouette-stack image. The inputs for the procedure are a set of silhouette images $\{O_{j,1}, \dots, O_{j,M}\}$, a set of axis-of-symmetry hypotheses represented by query angles $\Theta = \{ \theta_1, \dots, \theta_{N_\theta} \}$, and a user-set cutoff frequency $\tau$ for the amplitude spectrum, described below. First, individual silhouette images $O_{j,k}$ are co-added to obtain the stacked image $\bar{O}_j$. Next, the DFT amplitude spectrum of $\bar{O}_j$, $A$, is computed.

To reduce the effect of image noise on symmetry detection, a low-pass filter is then applied to the amplitude spectrum in order to preserve low-frequency components, associated with the primary symmetry information, while eliminating higher-frequency noise. In the amplitude-spectrum image, frequencies increase from the center outward, hence applying a low-pass filter is equivalent to creating a circular crop of the amplitude-spectrum image, as shown in Definition \ref{def:circ_cropping}. The circular crop only preserves pixels whose radial distance from the image center is lower than the maximum-frequency threshold, $\tau$.

Next, the dynamic range of the squared amplitude spectrum is compressed using a logarithmic transformation (see State \ref{state:log}, Algorithm \ref{alg:in-plane}). This compression step amplifies symmetric patterns while preserving the spectrum's fundamental symmetry properties through its monotonic relationship with the original amplitude values. Lastly, the symmetry score $\psi$ is evaluated across all input query angles. The pole-projection angle estimate $\hat{\alpha}$ corresponds to the query angle $\theta_i$ associated with the maximum symmetry score.


\section{Experimental Results}
\label{sec:results}

We employ numerical simulations to assess performance of the PoleStack algorithm. Section \ref{sec:in_plane_est} describes the simulation setup and results associated with pole-projection angle estimation, whereas Section \ref{sec:out_of_plane_est} evaluates 3D-pole estimation performance based on pole-triangulation.

\subsection{Pole-projection Angle Estimation}
\label{sec:in_plane_est}

We test the proposed approach for pole-projection angle estimation (Algorithm \ref{alg:in-plane}) using synthetic images of previously visited small celestial bodies. In the simulation, the observed body completes one full rotation about its pole, set as an arbitrary fixed direction with respect to the inertial frame. Observations are collected from a hovering-camera (i.e., constant-latitude and range, see Section \ref{sec:hovering_camera_model}) viewset. Images are simulated using the small body-rendering tool based on Blender Cycles presented by Villa et al\cite{villa2023image}. The simulation setup is outlined in Table \ref{tab:params_images}. We generate two different image sets using available shape models from asteroid 101955 Bennu\cite{lauretta2019unexpected} and comet 67P Churyumov-Gerasimenko (67P/C-G)\cite{thomas2015morphological}. Bennu exhibits a diamond-shaped structure with a high level of axial symmetry, whereas comet 67P is a bilobed, asymmetrical object.

We test the algorithm by simulating a large sun-phase angle (90 degrees), to assess robustness to challenging lighting conditions and substantial symmetry reduction in the stacked silhouette due to self-shadowing. Images are based on orthographic camera projection to resemble long-range observations, usually encountered throughout the target-approach phase. Two example images from the simulated sets are shown in Figure \ref{fig:img_0}. Due to the high sun phase, the observed silhouettes are corrupted by surface-shadowing effects, in particular for the 67P case as one lobe casts shadows onto the other at certain observation geometries.

\begin{table}[]
    \centering
    \caption{Simulation setup for in-plane pole estimation}
    \begin{tabular}{|c|c|c|}
        \hline
        Parameter & Value & Comments\\
        \hline
        Imaged bodies & 101955 Bennu, 67P/C–G & Different Shapes\\
        $\phi_f-\phi_0$ & $360\,\mathrm{deg}$ & Camera longitude range\\
        $\phi_k-\phi_{k-1}$ & $1\,\mathrm{deg}$ & Longitude Increments Between Images\\
        $\lambda$ & $14\,\mathrm{deg}$ & Camera latitude \\
        Sun Phase & $90\,\mathrm{deg}$ & Challenging case \\
        $\alpha$ & $20\,\mathrm{deg}$ & True in-plane pole angle\\
        Image resolution & $1024\times 1024\,\mathrm{pixels}$ & \\
        Camera projection & Orthographic & Simulating long-range images (e.g., approach)\\
        $\theta_i - \theta_{i-1}$ & $1\,\mathrm{deg}$ & Query-angle increments (see Algorithm \ref{alg:in-plane})\\
        $\tau$ & 100 pixels & Cutoff frequency for DFT amplitude spectrum\\
        Image rotation method & Nearest Neighbor & \\
        \hline
    \end{tabular}
    \label{tab:params_images}
\end{table}

\begin{figure*}[t!]
    \centering
    \begin{subfigure}[t]{0.4\textwidth}
        \centering
        \includegraphics[width=\textwidth, keepaspectratio]{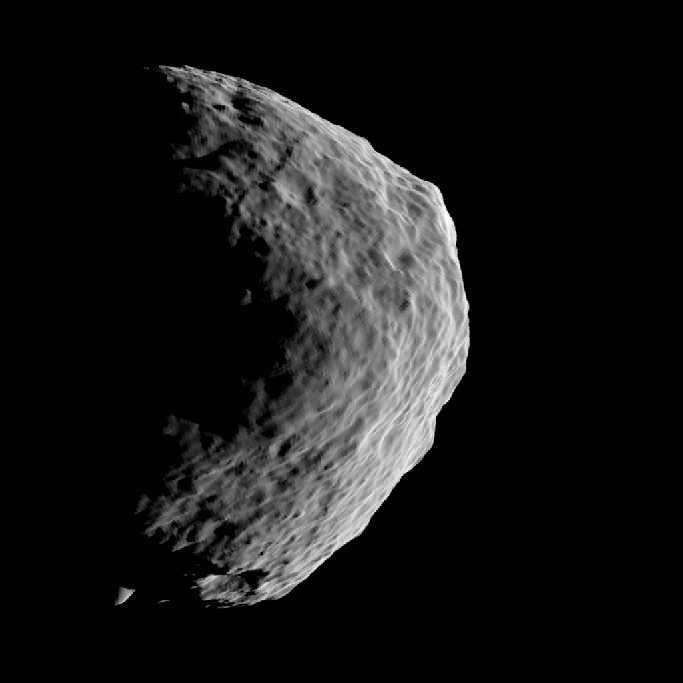}
        \caption{Asteroid Bennu}
    \end{subfigure}%
    ~ 
    \begin{subfigure}[t]{0.4\textwidth}
        \centering
        \includegraphics[width=\textwidth, keepaspectratio]{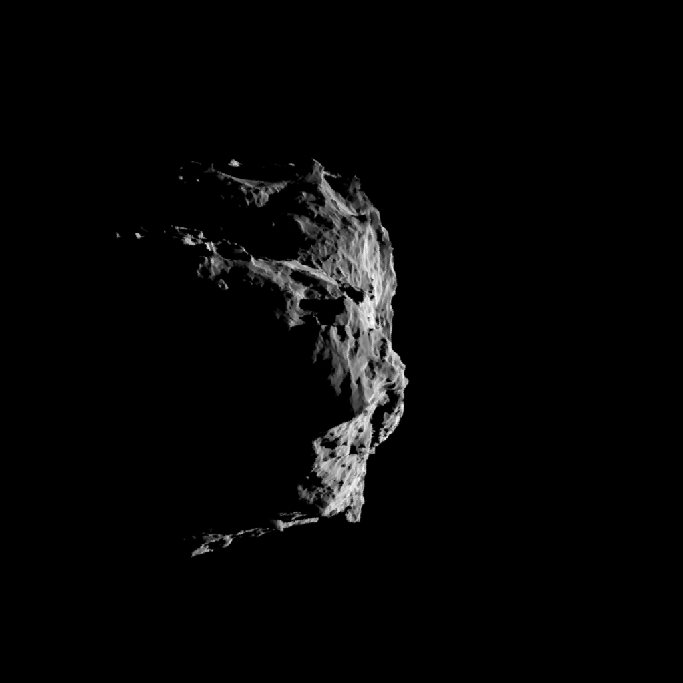}
        \caption{Comet 67P}
    \end{subfigure}
    \caption{First image from the two simulated image sets used to test pole-projection angle estimation. The object silhouettes are significantly corrupted by surface shadowing, due to the selected high sun phase (90 degrees). Images are centered around the computed brightness centroid.}
    \label{fig:img_0}
\end{figure*}

For each image set, we present results based on two cases: (1) assuming that silhouette images and the corresponding center-of-mass locations are perfectly aligned with each other, referred to as the perfect-alignment case; (2) stacking silhouette images by aligning the computed center-of-brightness locations, i.e., the brightness-centroid alignment case. We use the classical moment algorithm to compute the center of brightness\cite{owen2011methods}. Figure \ref{fig:ampl_spectr} shows the silhouette-stack images and the corresponding DFT amplitude spectra obtained for such two cases. The effect of centerfinding errors is mostly noticeable in the frequency domain. Interestingly, the errors introduced by centerfinding produce a smoothing effect on artifacts in the frequency domain, while preserving the overall symmetry about the pole direction. This phenomenon can be explained by the dilution of sharp foreground edges in the silhouette stack caused by small, random-like translational offsets between individual silhouettes. The amplitude spectra exhibit higher energy and structure at low frequency, motivating the use of a low-pass filter to improve performance. 


\begin{figure*}[htbp!]
    \centering
    \begin{subfigure}[t]{0.3\textwidth}
        \centering
        \includegraphics[width=\textwidth, keepaspectratio]{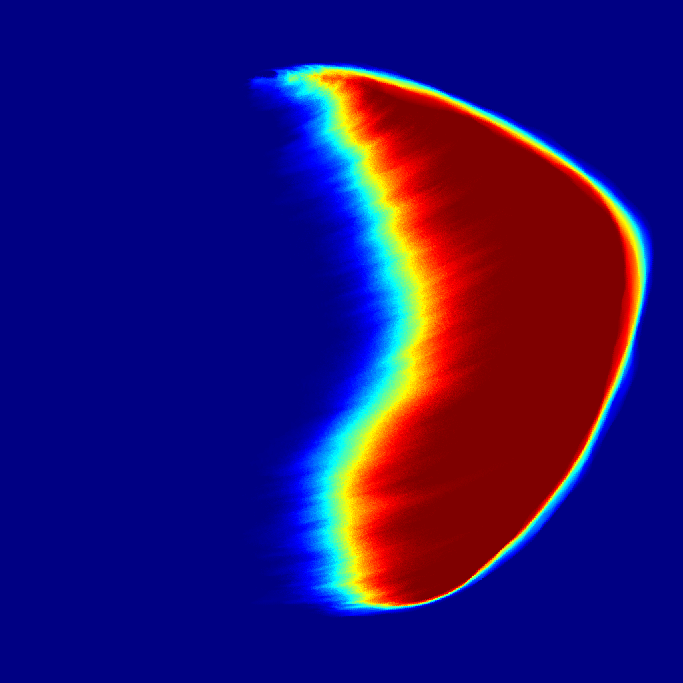}
    \end{subfigure}%
    ~ 
    \begin{subfigure}[t]{0.3\textwidth}
        \centering
        \includegraphics[width=\textwidth, keepaspectratio]{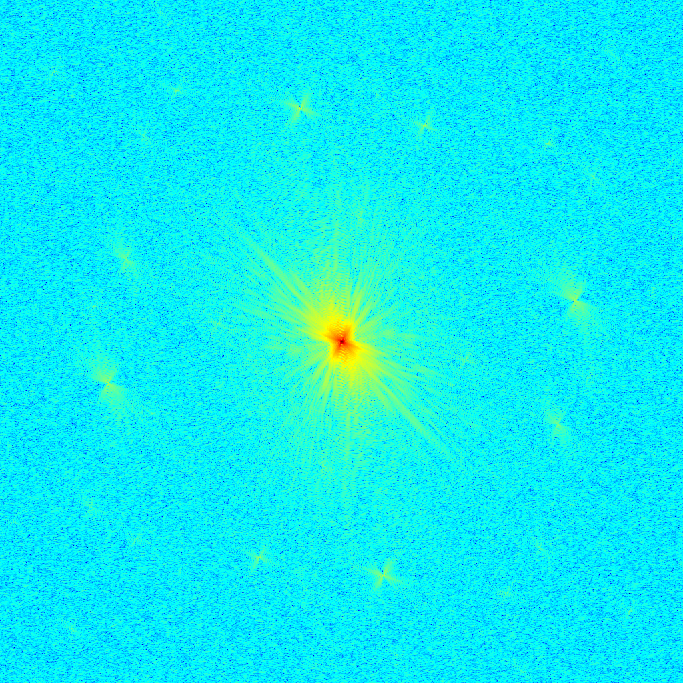}
    \end{subfigure}%
    ~
    \begin{subfigure}[t]{0.3\textwidth}
        \centering
        \includegraphics[width=\textwidth, keepaspectratio]{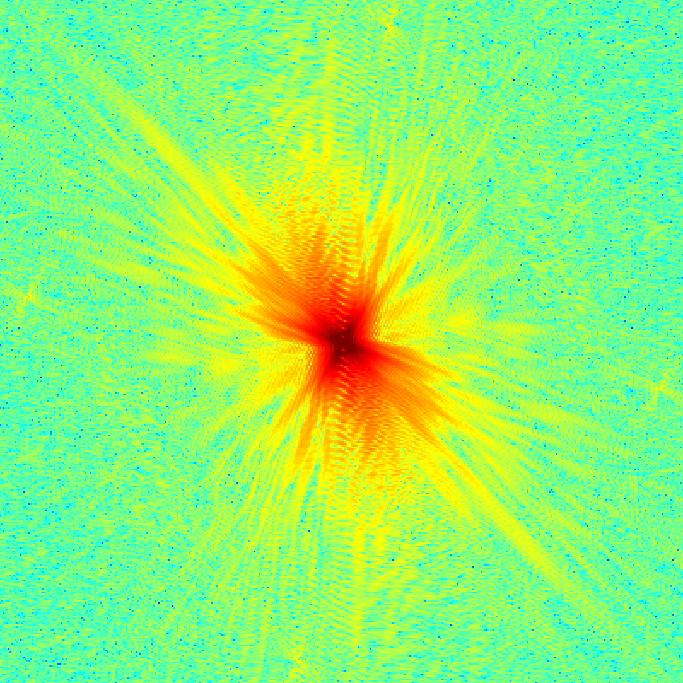}
    \end{subfigure}

    \begin{subfigure}[t]{0.3\textwidth}
        \centering
        \includegraphics[width=\textwidth, keepaspectratio]{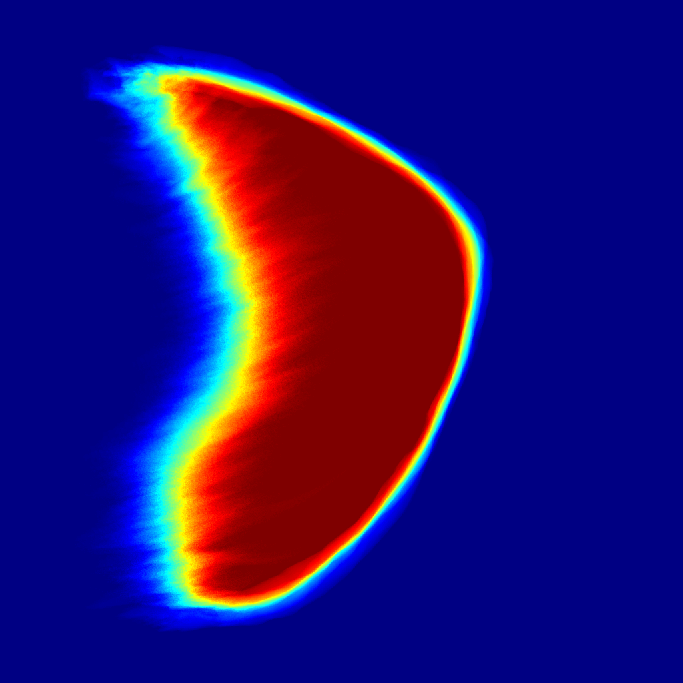}
    \end{subfigure}%
    ~ 
    \begin{subfigure}[t]{0.3\textwidth}
        \centering
        \includegraphics[width=\textwidth, keepaspectratio]{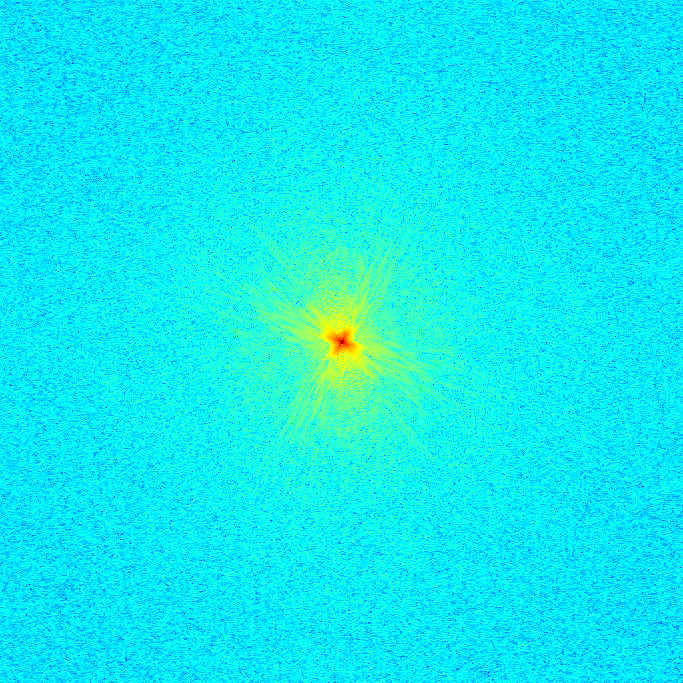}
    \end{subfigure}%
    ~
    \begin{subfigure}[t]{0.3\textwidth}
        \centering
        \includegraphics[width=\textwidth, keepaspectratio]{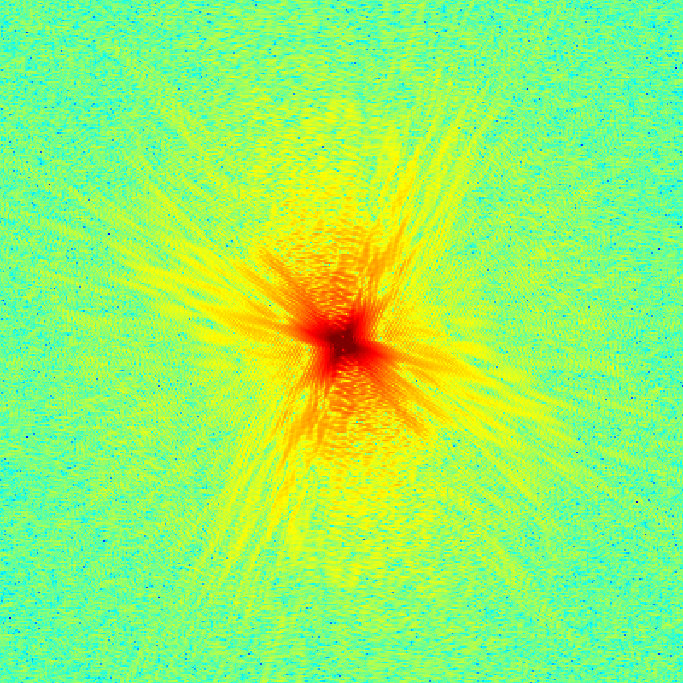}
    \end{subfigure}

    \begin{subfigure}[t]{0.3\textwidth}
        \centering
        \includegraphics[width=\textwidth, keepaspectratio]{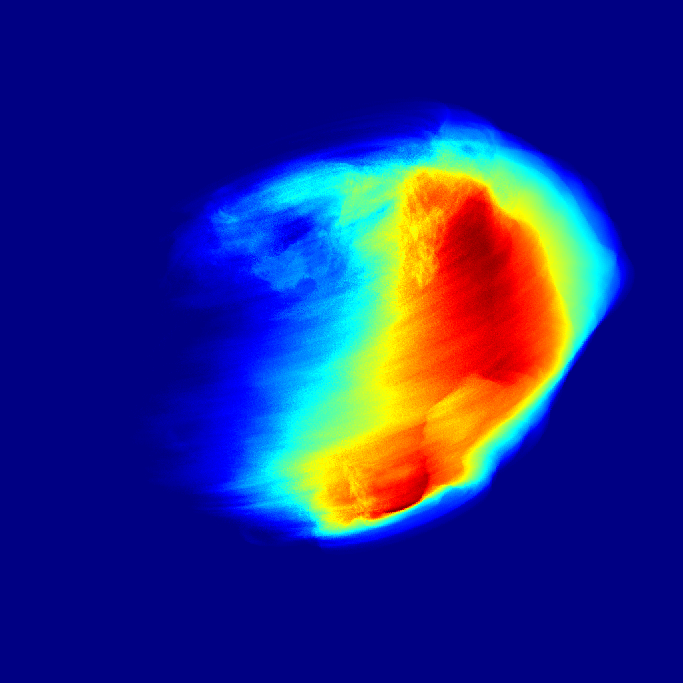}
    \end{subfigure}%
    ~ 
    \begin{subfigure}[t]{0.3\textwidth}
        \centering
        \includegraphics[width=\textwidth, keepaspectratio]{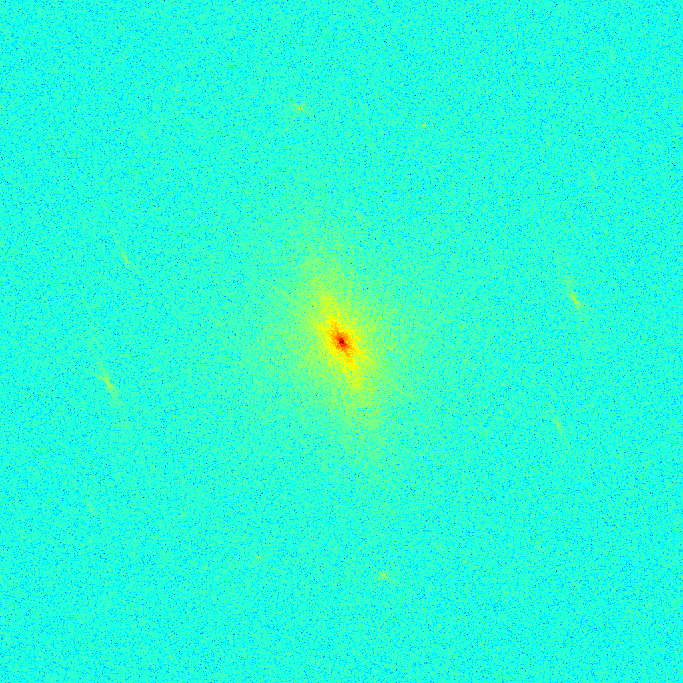}
    \end{subfigure}%
    ~
    \begin{subfigure}[t]{0.3\textwidth}
        \centering
        \includegraphics[width=\textwidth, keepaspectratio]{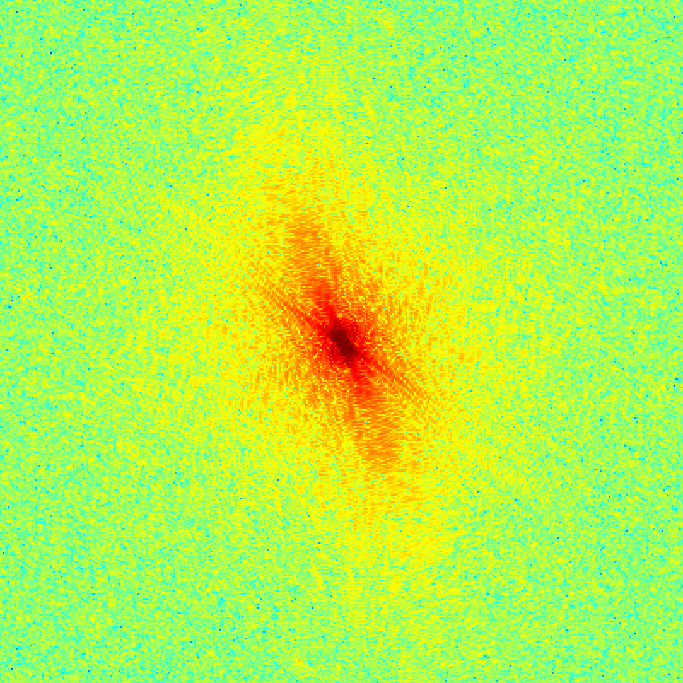}
    \end{subfigure}

    \begin{subfigure}[t]{0.3\textwidth}
        \centering
        \includegraphics[width=\textwidth, keepaspectratio]{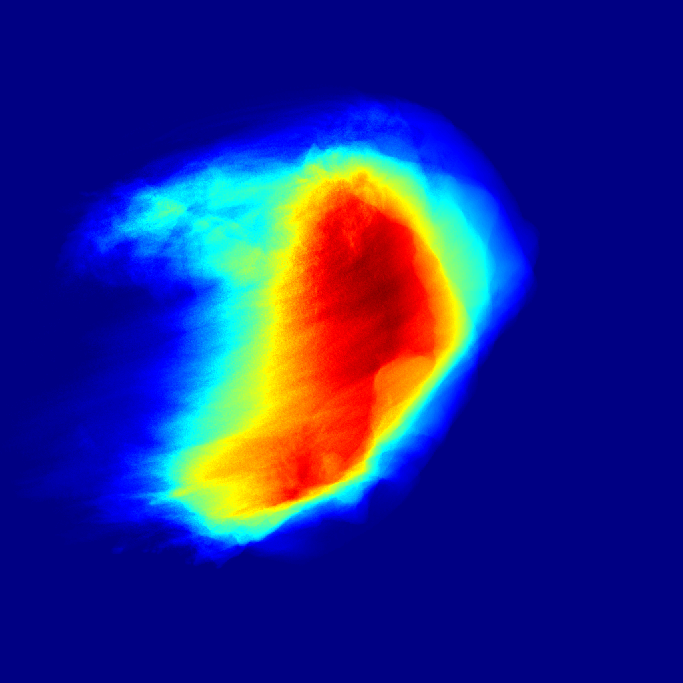}
        \caption{Silhouette-stack image ($\bar{O}$)}
    \end{subfigure}%
    ~ 
    \begin{subfigure}[t]{0.3\textwidth}
        \centering
        \includegraphics[width=\textwidth, keepaspectratio]{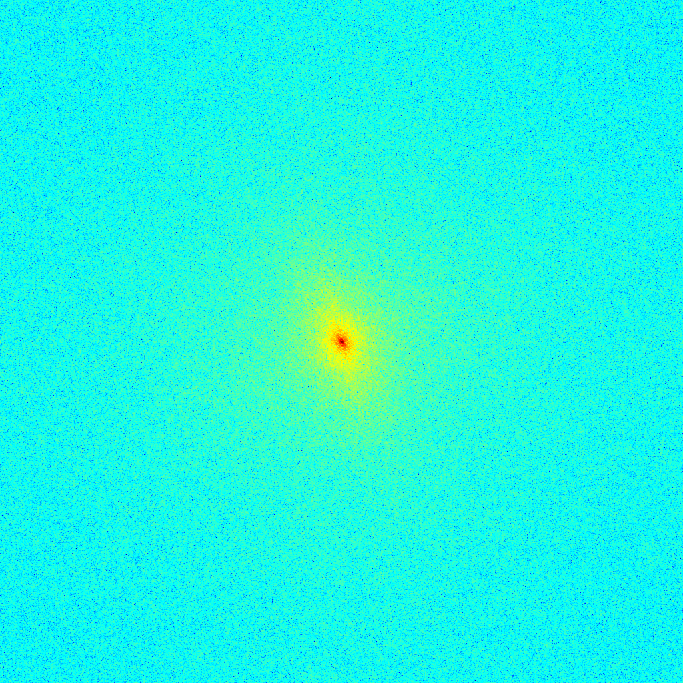}
        \caption{Log-power amplitude spectrum ($\mathrm{log}(1+A^2)$)}
    \end{subfigure}%
    ~
    \begin{subfigure}[t]{0.3\textwidth}
        \centering
        \includegraphics[width=\textwidth, keepaspectratio]{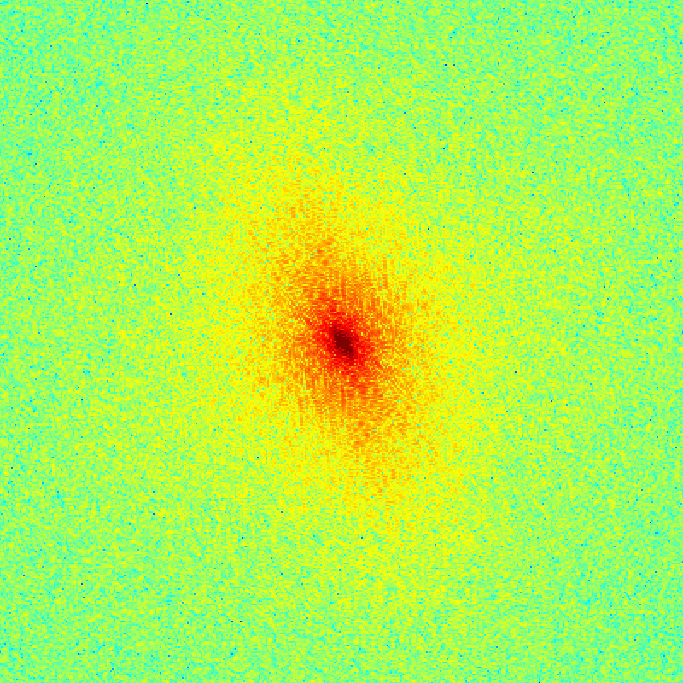}
        \caption{Compressed amplitude spectrum magnified about the center ($200\times 200$-pixel window); image clipping is used to facilitate visualization.}
    \end{subfigure}

    \caption{Silhouette-stack images in the spatial (left) and frequency (center and right) domain for different case studies: Bennu, perfect alignment (row 1); Bennu, brightness-centroid alignment (row 2); 67P, perfect alignment (row 3); 67P, brightness-centroid alignment (row 4).}
    \label{fig:ampl_spectr}
\end{figure*}

\begin{figure}
    \centering
    \includegraphics[width=0.7\linewidth]{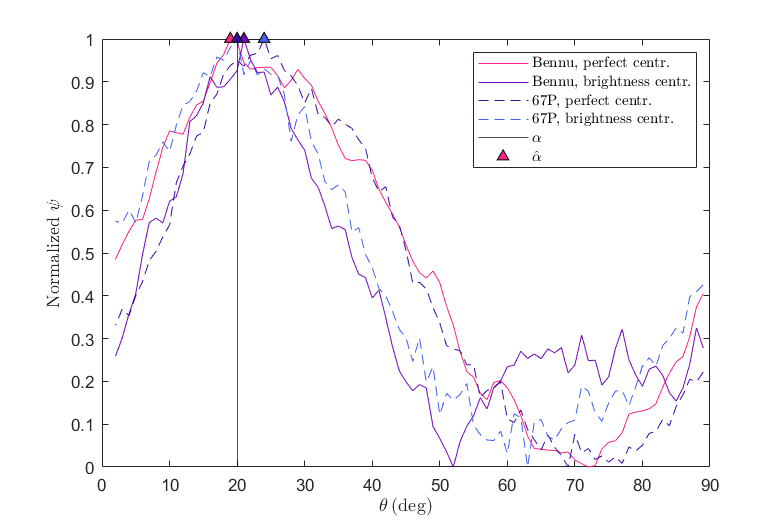}
    \caption{Pole-projection angle estimates (triangle markers) and associated normalized symmetry scores (lines). Error magnitudes span between $0\,\mathrm{deg}$ (Bennu, brightness-centroid alignment case) and $3\,\mathrm{deg}$ (67P, perfect-alignment case), up to angle-search discretization errors. The vertical line represents the true pole-projection angle.}
    \label{fig:in_plane_est}
\end{figure}

Estimation results for the pole-projection angle $\alpha$ are shown in Figure \ref{fig:in_plane_est}. For all cases, the symmetry score $\psi$ increases in the vicinity of the true value, $\alpha$. The resulting error magnitudes are between 0 deg and 3 deg, depending on the specific case. The largest error (3 deg) is obtained for the 67P, perfect-alignment case, which could be explained by the presence of more asymmetric structure as discussed in the previous paragraph. It should be noted that input query angles only span a single quadrant of the DFT amplitude spectrum, as discussed in Section \ref{sec:pole_ambig}.

\subsubsection{Effect of Data Volume Reduction}

In some cases, only low-resolution images, possibly collected across partial camera-longitude intervals, may be available. Hence, we study performance of the algorithm for pole-projection angle estimation when simultaneously reducing (1) image resolution, $N$, and (2) the camera-longitude range, $[\phi_0,\phi_f]$. We refer to the combined effect of such factors as ``data volume reduction". Image resolution is reduced to 25\% of the original value (see Table \ref{tab:params_images}), resulting in $256\times 256$-pixel frames; the camera-longitude range is constrained to a hemispherical arc, i.e., $[\phi_0,\phi_f]=[0,180]\,\mathrm{deg}$. For this experiment, the more realistic brightness-centroid alignment case is studied, and the cutoff frequency threshold is set to $\tau=126\,\mathrm{pixels}$ to preserve the entire DFT amplitude spectrum.

The resulting DFT amplitude spectra (Figure \ref{fig:ampl_resized}) show that the symmetry signal is qualitatively preserved, despite the decrease in data volume. Intriguingly, estimation errors (Figure \ref{fig:in_plane_est_resized}) are comparable to those obtained from higher image resolution and greater camera-longitude span.

\begin{figure*}[t!]
    \centering
    \begin{subfigure}[t]{0.4\textwidth}
        \centering
        \includegraphics[width=\textwidth, keepaspectratio]{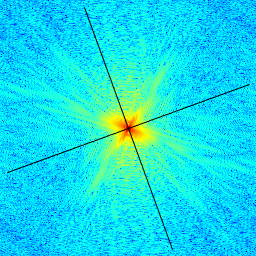}
        \caption{Asteroid Bennu}
    \end{subfigure}%
    ~ 
    \begin{subfigure}[t]{0.4\textwidth}
        \centering
        \includegraphics[width=\textwidth, keepaspectratio]{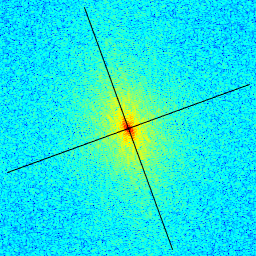}
        \caption{Comet 67P}
    \end{subfigure}
    \caption{Log-power amplitude spectra ($\mathrm{log}(1+A^2)$) for the reduced data volume case, at $256\times 256$-pixel resolution and observations collected across a hemispherical arc. The pole projection and its perpendicular axis are reported (black lines).}
    \label{fig:ampl_resized}
\end{figure*}

\begin{figure}
    \centering
    \includegraphics[width=0.7\linewidth]{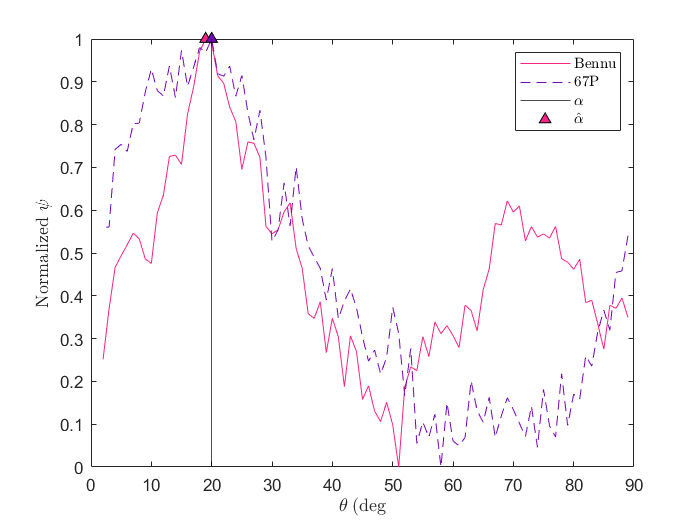}
    \caption{Pole-projection estimates for the reduced data volume case, using $256\times 256$-pixel images collected across a hemispherical arc. Pole-projection error magnitudes for the two cases are $1\,\mathrm{deg}$ (Bennu) and $0\,\mathrm{deg}$ (67P). Brightness-centroid alignment is employed for both cases.}
    \label{fig:in_plane_est_resized}
\end{figure}

\subsection{Pole Triangulation Results}
\label{sec:out_of_plane_est}

The proposed pole-triangulation method (see Section \ref{sec:pole_triangulation}) is studied and validated using a Monte Carlo simulation. We are interested in characterizing performance as a function of camera-viewpoint changes across observations. A total of $10^5$ Monte Carlo iterations are performed. For each iteration, the orientation of the true pole $\boldsymbol{\omega}$ and the camera axes associated with $\aleph$ camera views are randomly sampled from the unit sphere. Synthetic measurements of pole-projection angles $\{\hat{\alpha}_1, \dots, \hat{\alpha}_\aleph\}$ collected from the corresponding camera views are used as inputs for pole triangulation. These measurements are generated by perturbing the true pole-projection angle $\alpha$ by a randomly sampled, normally distributed error component with zero mean and variance $\sigma_\alpha^2$; random samples exceeding the $3\sigma_\alpha$ bounds are discarded and a new value is sampled. The pole estimate $\hat{\boldsymbol{\omega}}$ is computed using pole triangulation, according to Equations \ref{eq:A omega = 0} and \ref{eq: A = sin_alpha}. Results are presented for various values of $\sigma^2_\alpha$ and $\aleph$.

For each camera-view pair, viewpoint changes are quantified by the angular separation between camera-boresight axes, defined by the angle $\beta_{ij} \in [0,\pi]$ given by:

\begin{equation}
    \beta_{ij} = \mathrm{cos}^{-1}(\mathbf{k}_{\mathcal{C}_i}^\top \mathbf{k}_{\mathcal{C}_j}),\, i,j = 1,\dots,\aleph
\end{equation}

where $\mathbf{k}_{\mathcal{C}_i}$ and $\mathbf{k}_{\mathcal{C}_j}$ are the camera-boresight axes of the $i$-th and $j$-th views, respectively.

We quantify pole-estimation errors based as the angular separation $\epsilon\in[0,\pi]$ between the true ($\boldsymbol{\omega}$) and estimated ($\hat{\boldsymbol{\omega}}$) pole directions, i.e.:

\begin{equation}
    \epsilon = \mathrm{cos}^{-1}(\hat{\boldsymbol{\omega}}^\top\boldsymbol{\omega})
\end{equation}

Monte Carlo results are reported in Figures \ref{fig:FPE_triang_mc_sigma_1deg}-\ref{fig:FPE_triang_views}. Figure \ref{fig:FPE_triang_mc_sigma_1deg} shows pole estimation errors as a function of camera-view changes when using two camera views (i.e., $\aleph=2$) and $\sigma_\alpha=1^\circ$. This measurement noise value is representative of errors obtained through the analyses presented in Section \ref{sec:in_plane_est}. We find that $\epsilon$ is minimum around $\beta=\pi/2$, where the mean value of $\epsilon$ is close to $\sigma_\alpha$. This can be explained by the fact that a pair of orthogonal camera planes, and the associated pole-projection angle measurements, maximize the observability of the pole direction. Conversely, small angular separations between camera planes---i.e., for $\beta$ close to $0$ or $\pi$---lead to poor observability of the pole direction since pole projections are similar to each other. Importantly, mean-error results also exhibit a distinct knee at relatively low angular separation values, indicating a diminishing return in estimation accuracy for increasingly larger camera-view angular separations. In Figure \ref{fig:FPE_triang_mc_sigma_1deg}, the mean estimation error approaches the minimum value for $\beta \approx 20\,\mathrm{deg}$. Figure \ref{fig:FPE_triang_mc_sigmas} shows the mean estimation errors for various $\sigma_\alpha$. Although these errors generally increase with $\sigma_\alpha$ overall, the curves remain similar up to the highest tested value of $\sigma_\alpha=10^\circ$.

\begin{figure*}[t!]
    \centering
    \begin{subfigure}[t]{0.5\textwidth}
        \centering
        \includegraphics[width=\textwidth, keepaspectratio]{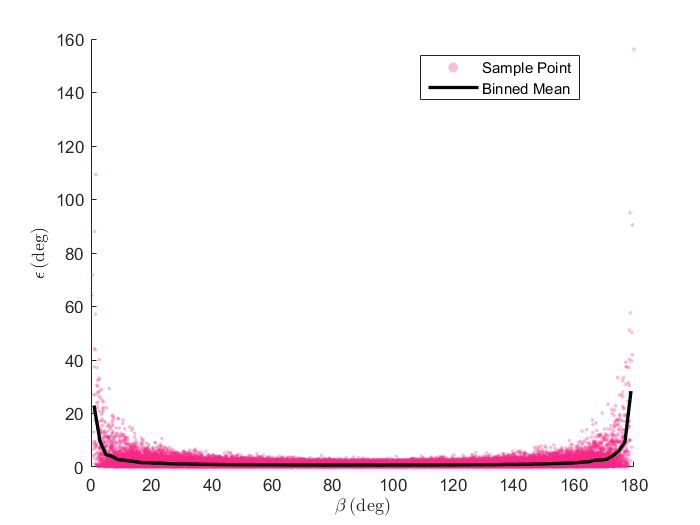}
        \caption{Original plot}
    \end{subfigure}%
    ~ 
    \begin{subfigure}[t]{0.5\textwidth}
        \centering
        \includegraphics[width=\textwidth, keepaspectratio]{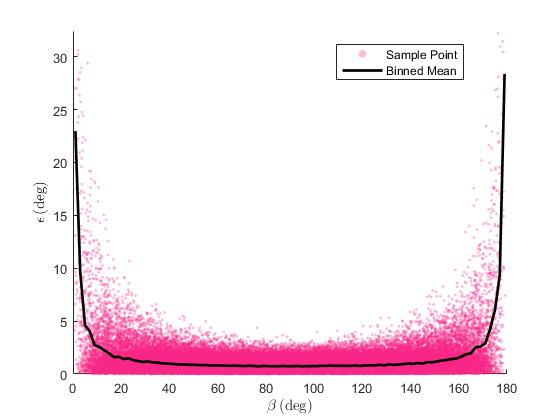}
        \caption{Magnified around mean values}
    \end{subfigure}
    \caption{Monte Carlo simulation results for the pole estimation error $\epsilon$ as a function of the camera-boresight angular separation $\beta$ between two views, given $\sigma_\alpha=1^\circ$. Each dot represents an individual Monte Carlo sample, and the solid line indicates the binned mean error within $2^\circ$ bins.}
    \label{fig:FPE_triang_mc_sigma_1deg}
\end{figure*}

\begin{figure}
    \centering
    \includegraphics[width=0.7\linewidth]{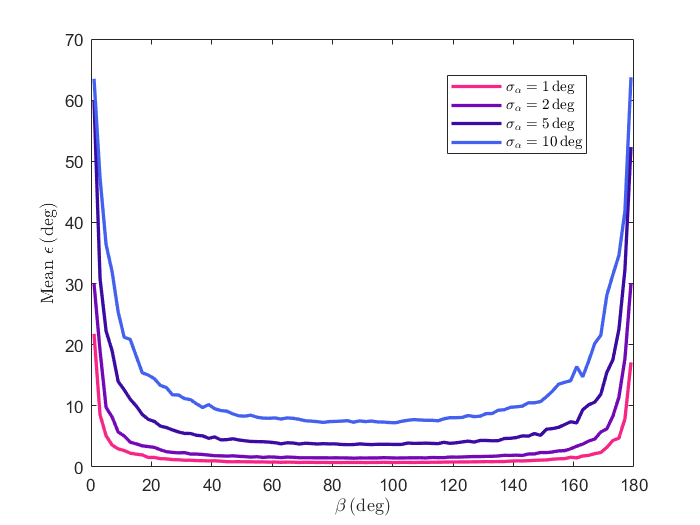}
    \caption{Monte Carlo simulation results for the binned mean ($2^\circ$ bins) of the pole estimation errors $\epsilon$ as a function of the boresight angular separation $\beta$, for different values of $\sigma_\alpha$.}
    \label{fig:FPE_triang_mc_sigmas}
\end{figure}

Lastly, Figure \ref{fig:FPE_triang_views} illustrates how the number of camera views $\aleph$ affects pole estimation performance, with each view randomly sampled as described earlier. The results show that increasing $\aleph$ reduces both the error variance and a the outlier rate, dropping from 1\% to 0.005\% when $\aleph$ increases from 2 to 4.

\begin{figure}
    \centering
    \includegraphics[width=0.7\linewidth]{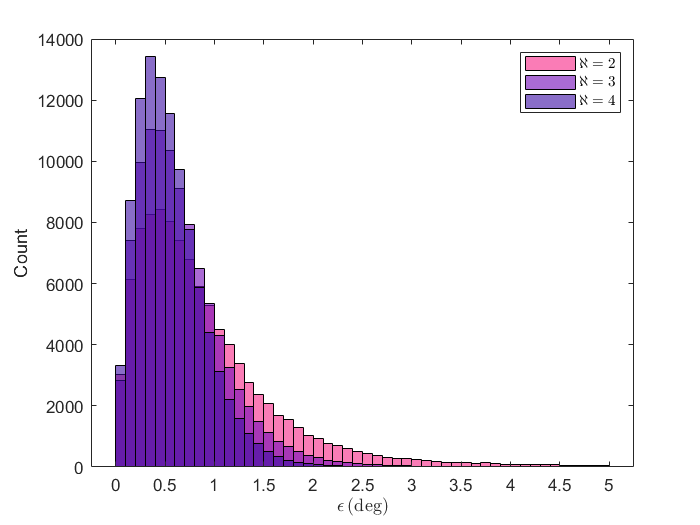}
    \caption{Monte Carlo simulation results showing estimation error histograms for various numbers of camera views ($\aleph$) at $\sigma_\alpha=1$. The horizontal axis is truncated at $5^\circ$ for clarity; the number of samples with errors exceeding $5^\circ$ is 1061 (1\%), 30 (0.03\%), and 5 (0.005\%), for $\aleph=2,3,4$, respectively.}
    \label{fig:FPE_triang_views}
\end{figure}

\section{Conclusions}
\label{sec:conclusions}

In this work, we present PoleStack, an algorithm for estimating the rotation pole of an irregularly shaped object by stacking silhouette observations from multiple images. We develop the theoretical framework to demonstrate that silhouette observations from a hovering camera exhibit reflective symmetry with respect to the object’s pole projection onto the image plane. As such, the projected-pole direction is estimated as the direction of maximum symmetry in the image obtained by stacking (i.e., co-adding) consecutive silhouette observations. By transforming the silhouette-stack image into the frequency domain, the resulting representation is invariant to translation, and thus knowledge of the object’s center of mass in the camera plane is not required for pole estimation. By triangulating multiple pole-projection measurements, the 3D pole orientation can then be estimated. Building upon these principles, we obtain a relatively simple algorithm which consists of a few steps.

We demonstrate the robustness of PoleStack against common error sources, including image-alignment errors (e.g., due to centroiding) and surface shadowing, both of which can degrade the symmetry in silhouette-stack images. Using numerical simulations, we achieve degree-level accuracy in pole-projection estimates for both a near-axisymmetric, diamond-like shape (asteroid Bennu) and a bilobed shape (comet 67P/C-G). In both cases, the silhouette observations are significantly affected by surface shadowing (collected at a $90^\circ$ sun phase angle). Furthermore, we find that pole-projection estimation performance remains largely unaffected when the camera-longitude arc is reduced (i.e., when observing partial rotations) and when using lower-resolution images.

The accuracy of the 3D pole direction estimation generally depends on the pole-projection measurement noise and the angular separation between hovering-camera views. Results indicate that degree-level 3D pole estimation accuracy can be achieved with an angular separation between camera boresight of $10^\circ$ or more. Such boresight deviation could be obtained by leveraging natural spacecraft-latitude changes that occur during the target-approach phase. If changes in the camera latitude are relatively slow with respect to the imaging cadence, portions of the approach trajectory could be well approximated by the hovering-camera model. Each portion could then be used to extract a pole-projection measurement which would contribute to the 3D-pole estimate. The pole estimate will continue to improve as additional hovering-camera image sets and measurements become available throughout the approach.

In future work, we will extend our approach to the rotation-axis estimation of tumbling objects. Furthermore, although this study focuses on small celestial bodies, the technique could also be applied to artificial objects such as uncooperative spacecraft.


\section*{Acknowledgments}
The authors would like to acknowledge the support of the Strategic University Research Partnerships Program at the Jet Propulsion Laboratory. A portion of this research was carried out at the Jet Propulsion Laboratory, California Institute of Technology, under a contract with the National Aeronautics and Space Administration. \copyright 2025 California Institute of Technology. Government sponsorship acknowledged.

Portions of this manuscript were revised using generative AI to improve readability and clarity, without altering the underlying content. The authors assume full responsibility for all final content.

\bibliography{sample}

\section*{Appendix}

\subsection*{Image Symmetries and Transformations}

\begin{definition}
\label{def:reflective_transformation}
    Let $\mathbf{x}\in\mathbb{R}^2$ be a 2D vector. The reflective transformation of $\mathbf{x}$ with respect to the line passing by the origin and with an angle $\theta$ with respect to the horizontal axis is a matrix $\mathrm{Ref}(\theta)\in\mathbb{R}^{2\times 2}$ given by:

    \begin{equation}
        \mathrm{Ref}(\theta) =
        \begin{bmatrix}
            \mathrm{cos}(2\theta) & \mathrm{sin}(2\theta)\\
            \mathrm{sin}(2\theta) & -\mathrm{cos}(2\theta)\\
        \end{bmatrix}
    \end{equation}

    The corresponding reflected point $\mathbf{x}_{\mathrm{Ref},\theta}$ is such that:

    \begin{equation}
        \mathbf{x}_{\mathrm{Ref},\theta} = \mathrm{Ref}(\theta) \, \mathbf{x}.
    \end{equation}
\end{definition}

\begin{definition}
\label{def:image_refl_symm_theta}
Let $I\in\mathbb{R}^{N\times N}$ be an image. Then, $I$ exhibits reflective symmetry with respect to $\theta$ if and only if

\begin{equation}
    I(m_{\mathrm{Ref},\theta},n_{\mathrm{Ref},\theta}) = I(m,n),\, m,n = 1,\dots, N
\end{equation}

where $m_{\mathrm{Ref},\theta},n_{\mathrm{Ref},\theta}$ are the pixel indices corresponding to the reflection of the $mn$-th pixel, i.e.:

\begin{equation}
\begin{bmatrix}
u_{n_{\mathrm{Ref}(\theta)}}\\
v_{m_{\mathrm{Ref}(\theta)}}\\
\end{bmatrix} = \mathrm{Ref}(\theta)
\begin{bmatrix}
    u_n \\ v_m
\end{bmatrix}.
\end{equation}
    
\end{definition}

\begin{definition}
\label{def:vert_refl_symm}
An image $I\in\mathbb{R}^{N\times N}$ is said to exhibit vertical reflective symmetry if and only if:
\begin{equation}
I(m,n) = I(m,N-n+1),\; m,n=1,\dots,N.
\end{equation}

\end{definition}

\begin{definition}
    \label{def:img_function}
    Let $I\in\mathbb{R}^{N\times N}$ be an image. We define $f_I: \mathbb{Z}^2 \rightarrow \mathbb{R}$ as the discrete function representing $I$, such that

    \begin{equation}
        f_I(m,n) = \begin{cases}
            I(m,n) & \text{if} \;\; m,n \leq N\\
            0 & \text{otherwise}
        \end{cases}
    \end{equation}
    
\end{definition}

\begin{definition}
\label{def:interp_fctn}
    We define $g_I: \mathbb{R}^2 \rightarrow \mathbb{R}$ as an interpolating function (e.g., nearest neighbor, bilinear, or bicubic) of $f_I$ (Definition \ref{def:img_function}), such that

    \begin{equation}
        g_I(u_n,v_m) = f_I(m,n),\, \, m,n = 1,\dots,N
    \end{equation}

\end{definition}

\begin{definition}
    \label{def:I_theta}
    Let $I\in\mathbb{R}^{N\times N}$ be an image and $R(\theta) \in \mathbb{R}^{2\times 2}$ a 2D rotation by the angle $\theta$, given by

    \begin{equation}
        R(\theta) =
        \begin{bmatrix}
            \mathrm{cos}(\theta) & -\mathrm{sin}(\theta)\\
            \mathrm{sin}(\theta) & \mathrm{cos}(\theta)
        \end{bmatrix}
    \end{equation}
    
 We define the image-rotation operator, $\mathrm{Rot}$, as the transformation which rotates the image $I$ according to $R(\theta)$, such that the rotated image $I_\theta = \mathrm{Rot}(I; \theta),\, I_\theta \in \mathbb{R}^{N'_r \times N'_c}$ is given by:

    \begin{equation}
        I_\theta(m,n) = g_I(u_{n,\theta},v_{m,\theta}),\, m=1,\dots,V,\, n=1,\dots,W
  \end{equation}

  where $N'_r$ and $N'_c$ indicate the number of rows and columns of the bounding region in $I_\theta$ containing nonzero pixels, $g_I$ is an interpolation function (Definition \ref{def:interp_fctn}) and $(u_{n,\theta},v_{m,\theta})$ are image coordinates given by

  \begin{equation}
  \label{eq:u_m,v_n}
\begin{bmatrix}
u_n \\
v_m
\end{bmatrix} = R(\theta)
\begin{bmatrix}
u_{n,\theta} \\
v_{m,\theta}
\end{bmatrix}.
\end{equation}

\end{definition}

Due to the discrete nature of images, image rotation requires an interpolation procedure to assign pixel intensities to the rotated image. Multiple interpolation techniques exist, such as bilinear and bicubic methods, which make use of $2\times 2$ and $4\times 4$ pixel neighborhoods, respectively\cite{castleman1996digital}. Further, notice that the total size of the rotated image may differ from the original image size due to potential changes in the horizontal and vertical components introduced by rotation. 





\begin{definition}
    \label{def:I_ref}
    Given an image $I \in \mathbb{R}^{N\times N}$, its reflection about the vertical axis is an image $\underline{I}\in\mathbb{R}^{N\times N}$ such that

    \begin{equation}
        \underline{I}(m,n) = I(m,N-n+1),\,m,n=1,\dots,N
    \end{equation}
\end{definition}

\begin{definition}
\label{def:circ_cropping}

    Let $I\in\mathbb{R}^{N_r\times N_c}$ be a rectangular image with $N_r$ rows and $N_c$ columns, represented by the discrete function $f_I$, according to Definition \ref{def:img_function}. We define the circular-cropping operator, $\mathrm{Crop}$, as the transformation producing a cropped, square image $I_{\mathrm{crop},\tau} = \mathrm{Crop}(I; \tau),\, I_{\mathrm{crop},\tau} \in \mathbb{R}^{N'\times N'}$, such that the discrete function $f_{I_{\mathrm{crop},\tau}}$ representing $I_{\mathrm{crop},\tau}$ is given by: 

    \begin{equation}
        f_{I_{\mathrm{crop},\tau}}(m,n) = \begin{cases}
            f_I(m,n), & \text{if} \;\; \| [u_n, v_m] \| \leq \tau \\
            0, & \text{otherwise}
        \end{cases}
    \end{equation}

    where $\tau$ is a cropping parameter defining the radius of the circular crop.
    
\end{definition}

\subsection*{The Discrete Fourier Transform}

The Discrete Fourier Transform (DFT) is a computational method that transforms a finite set of evenly spaced data points into a corresponding set of complex numbers. The latter represent the signal's composition in terms of frequency components. The DFT can be generalized to multiple dimensions. For image processing, we are interested in the 2D DFT. Notably, the Fast Fourier Transform (FFT) is an efficient algorithm to compute the DFT of a sequence, used in a variety of applications such as signal and image processing\cite{brigham1988fast}.

\begin{definition}
    \label{def:dft}
    Let $I\in \mathbb{R}^{N\times N}$ be a real-valued image. The Discrete Fourier Transform (DFT) of $I$ is the $N\times N$ complex-valued image $F=\mathcal{F}(I),\, F \in \mathbb{C}^{N\times N}$ such that:

\begin{equation}
\label{eq:F_xy}
    F(x,y) = \sum_{m=1}^{N} \sum_{n=1}^{N}I(m,n) \, \mathrm{exp}\left[ -2\pi i \left( \frac{xm}{N} + \frac{yn}{N}\right) \right].
\end{equation}

 where $i$ in Equation \ref{eq:F_xy} denotes the imaginary units. The DFT amplitude spectrum is the image $A \in \mathbb{R}^{N\times N}$ defined as
 
\begin{equation}
    A(x,y) = |F(x,y)|.
\end{equation}

whereas the DFT phase spectrum is the image $F_\mathrm{ph} \in \mathbb{R}^{N\times N}$ defined as:

\begin{equation}
F_\mathrm{ph}(x,y) = \mathrm{arctan}\left( \dfrac{\mathrm{Im}(F(x,y)}{\mathrm{Re}(F(x,y))} \right) 
\end{equation}

where $\mathrm{Re(F)}$ and $\mathrm{Im(F)}$ denote the real and imaginary parts of $F$, respectively.
\end{definition}

The 2D DFT decomposes an image into the sum of complex exponential functions. It is closely related to the 1D DFT, where a 1D signal is broken down into a finite set of sinusoidal functions, each with a different amplitude and phase. In fact, the 2D DFT is equivalent to computing the 1D DFT of each image row to produce an intermediate image, and then computing the 1D DFT of each column of the latter image. (The order of the row-wise and column-wise 1D DFT operations does not matter.) The information content in the frequency-domain image is the same as in the space-domain image and the latter can be obtained from the former by applying the so-called inverse DFT.

The following property of the DFT is provided here without proof.

\begin{lemma}
\label{lemma:dft_linearity}
The DFT is a linear operator. That is, given two images $I_a,I_b\in\mathbb{R}^{N\times N}$ and two parameters $\zeta_a,\zeta_b\in\mathbb{R}$, we have:

\begin{equation}
    \mathcal{F}(\zeta_a I_a + \zeta_b I_b) = \zeta_a F_a + \zeta_b F_b
\end{equation}

where $F_a = \mathcal{F}(I_a)$ and $F_b = \mathcal{F}(I_b)$.
\end{lemma}


    

\end{document}